%% file: paper.tex
\documentclass{article}

\usepackage{microtype}
\usepackage{graphicx}
\usepackage{caption}
\usepackage{subcaption}

\usepackage{booktabs} %

\usepackage{hyperref}
\usepackage{enumitem}

\usepackage{algorithm,algpseudocode}

\usepackage[accepted]{icml2023}

\usepackage{amsmath}
\usepackage{amssymb}
\usepackage{mathtools}
\usepackage{amsthm}
\usepackage{bbm}

\usepackage[textsize=tiny]{todonotes}

\usepackage{multirow}
\usepackage{mathtools}
\usepackage{adjustbox}
\usepackage{gensymb}
\usepackage{natbib}

\usepackage[capitalize,noabbrev]{cleveref}

\theoremstyle{plain}
\newtheorem{theorem}{Theorem}[section]
\newtheorem{proposition}[theorem]{Proposition}

\theoremstyle{definition}
\newtheorem{definition}[theorem]{Definition}
\newtheorem{assumption}[theorem]{Assumption}
\theoremstyle{remark}

\DeclareMathOperator*{\argmin}{arg\,min}
\DeclareMathOperator*{\argmax}{argmax}

\usepackage{xspace}
\newcommand*{\eg}{e.g.\@\xspace}
\newcommand*{\ie}{i.e.\@\xspace}

\newcommand{\x}{\mathbf{x}}
\newcommand{\z}{\mathbf{z}}

\newcommand{\h}{\mathbf{h}}
\newcommand{\w}{\mathbf{w}}

\makeatletter
\newcommand*{\etc}{%
    \@ifnextchar{.}%
        {etc}%
        {etc.\@\xspace}%
}

\usepackage{xcolor}
\usepackage[normalem]{ulem} %
\newcommand\hl{\bgroup\markoverwith
  {\textcolor{yellow}{\rule[-.5ex]{2pt}{2.5ex}}}\ULon}

\icmltitlerunning{Great Models Think Alike: Improving Model Reliability via Inter-Model Latent Agreement}

\begin{document}

\twocolumn[
\icmltitle{Great Models Think Alike: Improving Model Reliability via Inter-Model Latent Agreement}

\begin{icmlauthorlist}
\icmlauthor{Ailin Deng}{sch}
\icmlauthor{Miao Xiong}{yyy}
\icmlauthor{Bryan Hooi}{sch,yyy}
\end{icmlauthorlist}

\icmlaffiliation{yyy}{Institute of Data Science, National University of Singapore, Singapore}
\icmlaffiliation{sch}{School of Computing, National University of Singapore, Singapore}

\icmlcorrespondingauthor{Ailin Deng}{ailin@u.nus.edu}

\icmlkeywords{Machine Learning, ICML}

\vskip 0.3in
]

\printAffiliationsAndNotice{}  %

\begin{abstract}
\input{000abstract}

\end{abstract}

\section{Introduction}
\label{sec:intro}
\input{010introduction}

\section{Related Work}
\label{sec:background}
\input{020background}

\section{Proposed Method}
\label{sec:meth}
\input{030method}

\section{Experiments}
\label{sec:exp}

\input{040experiment}

\section{Theoretical Analysis}
\label{sec:theory}
\input{theory.tex}

\section{Conclusions}
\label{sec:concl}
\input{050conclusion}

\section*{Acknowledgements}
This research is supported by the National Research Foundation Singapore under its AI Singapore Programme (Award Number: [AISG2-TC-2021-002]).

\bibliographystyle{icml2023}

\input{paper.bbl}
\newpage
\onecolumn

\input{060appendix}

\end{document}

%% file: 000abstract.tex
Reliable application of machine learning is of primary importance to the practical deployment of deep learning methods. A fundamental challenge is that models are often unreliable due to overconfidence~\cite{hendrycks17baseline}. In this paper, we estimate a model's reliability by measuring \emph{the agreement between its latent space, and the latent space of a foundation model}. However, it is challenging to measure the agreement between two different latent spaces due to their incoherence, \eg, arbitrary rotations and different dimensionality. To overcome this incoherence issue, we design a \emph{neighborhood agreement measure} between latent spaces and find that this agreement is surprisingly well-correlated with the reliability of a model's predictions. Further, we show that fusing neighborhood agreement into a model's predictive confidence in a post-hoc way significantly improves its reliability. 
Theoretical analysis and extensive experiments on failure detection across various datasets verify the effectiveness of our method on both in-distribution and out-of-distribution settings.

%% file: 010introduction.tex
Model reliability is a critical and challenging issue in deep neural networks for deploying neural network systems in real-world applications, particularly in safety-critical domains such as autonomous driving and medical diagnosis. In particular, a key challenge is that models tend to be overconfident~\cite{guo2017calibration,hendrycks17baseline,Ovadia2019CanYT}, and it is hard to identify   such overconfidence purely based on the model's own internal states, as these internal states are themselves potentially unreliable. A possible answer to this dilemma comes from the recent emergence of powerful general purpose (or `foundation') models~\cite{brown2020language, bommasani2021opportunities,radford2021learning}, which provide rich ``implicit knowledge" while being freely available for use without additional training costs. This presents an opportunity to use them to assist in evaluating the reliability of a newly trained model.

As encapsulated by the phrase ``great minds think alike", it is intuitive that a model tends to be more reliable on some input if its reasoning aligns well with that of other models. For example, if we want to examine whether a human learner has well-understood some input image (e.g. an animal), we can ask them questions about it (e.g. what animals is it similar to?): the more they agree with other human learners, the more confident we can be that they have correctly understood the image. Similarly, for models, we want to evaluate the reliability of a model by estimating the extent to which its reasoning agrees with that of a foundation model. This further leads to the main challenge: how can we quantitatively measure how much two models agree on a concept? In this work, we propose to measure this model agreement via latent spaces. 
 Intuitively, the two models agree if their latent spaces ``model the concept similarly"; however, this is complicated by the fact that different models are typically trained with different data distributions, model architectures and optimization objectives, leading to incoherence between different latent spaces: e.g. differing by an arbitrary rotation, and different dimensionalities.

To solve this problem, we introduce \emph{inter-model latent agreement} - a framework for measuring of how much two models agree on a sample while avoiding the incoherence issue, and then show its utility for estimating and enhancing model reliability. Our framework uses the similarity of neighborhoods in the latent spaces of an input as a proxy task to measure agreement between latent spaces of different models.

In particular, we present our main empirical observation that the reliability (or probability of correctness) of a model on a sample correlates well with its inter-model latent agreement with foundation models on that sample. Motivated by this, we propose our latent space agreement framework, which exploits the inter-model agreement to improve prediction reliability in a post-hoc way. Notably, our proposed inter-model latent agreement can be measured with only unlabeled samples, making it more broadly applicable. We further verify the effectiveness of our framework by conducting extensive experiments on failure detection over various datasets and provide theoretical analysis for our method.

Overall, the contributions and benefits of our approach are as follows:
\begin{itemize}[noitemsep,nolistsep]
    \item (Empirical Findings) We show that the inter-model agreement highly correlates with classification accuracy, suggesting the value of latent space agreement to improve a newly trained model's reliability.
    \item (Generality) We propose a general framework that enables using any foundation model via latent spaces in a post-hoc way without any fine-tuning to improve the predictive reliability of a newly trained model.
    \item (Effectiveness) We quantitatively verify the performance of our framework on failure detection across various datasets, including large-scale in-distribution (ID) and out-of-distribution (OOD) datasets, and provide further exploration, empirical and theoretical justification of the framework.

\end{itemize}

%% file: 020background.tex
\subsection{Failure Detection}
The main goal of failure detection is to predict whether a trained classifier will make an error on a test sample~\cite{jaeger2023a}. \citealt{hendrycks17baseline} propose maximum softmax probabilty (MSP), which directly uses the softmax predictions of the trained model. Follow-up works propose other uncertainty measures from a trained model, based on Monte Carlo Dropout or aggregated from multiple trained models, \eg, predictive entropy or variants~\cite{gal2016dropout,lakshminarayanan2017simple,liu2020energy}. \citealt{jiang2018trust} propose a distance ratio in the latent space of a classifier. In addition to detecting failures in in-distribution data, detecting failures under distribution shifts is a crucial problem to enhance model reliability in real-world applications, and has received increasing attention~\cite{hendrycks2019robustness,koh2021wilds,vaze2022openset}. Previous works have aimed to utilize internal information from a trained classifier~\cite{xiong2022birds,deng2022trust}. However, a classifier itself can be potentially unreliable~\cite{guo2017calibration,hein2019relu}, which motivates us to use external information to validate the reliability of a prediction and further enhance failure detection.

\subsection{Foundation Models}
The recent powerful foundation models~\cite{radford2021learning,bommasani2021opportunities} are pretrained on large-scale data and can provide rich ``implicit knowledge" to validate the prediction from a newly trained classifier. The prevalent paradigm to use these foundation models is fine-tuning with downstream data. However, traditional fine-tuning is often computationally intensive and requires a large amount of downstream labeled data, particularly as foundation models continue to grow in size. To alleviate these issues, recent methods propose to adapt the foundation models before use for some applications by prompting, instead of fully fine-tuning~\cite{bommasani2021opportunities}. Our method is different from the previous works as we propose to use the agreement between the latent space of a trained classifier and the latent space of a foundation model to validate a prediction, which requires no fine-tuning or adaptation.

%% file: 030method.tex
\subsection{Preliminaries}

Let $D = \{\mathbf{x}^{(i)}, y^{(i)}\}_{i=1}^n$, denote the training dataset containing $n$ samples, where $ \mathbf{x}^{(i)} \in  \mathbb{R}^m$ is the $i$-th input sample and $y^{(i)} \in \mathcal{Y} = \{ 1, \dots, C\}$ is the corresponding true class. A classification model consists of two parts: a feature extractor $B : \mathcal{X} \to \mathbb{R}^d$ and a linear head $f_w: \mathbb{R}^d \to \mathbb{R}^C$, parameterized by $w$. Given an input $\mathbf{x}$, the model produces a latent feature vector $\mathbf{z} = B(\mathbf{x})$ followed by the softmax probability output and predictive label:
\begin{align}
    \hat{P}(Y \mid \x, B, w ) = \mathsf{softmax}(f_w(\z)) \\
    \hat{y} = \argmax_{c \in \mathcal{Y}} \hat{P}(Y = c \mid \x, B, w).
\end{align}
Given an input $\x$, a foundation model can also produce a latent feature vector $\h = H (\x) \in \mathbb{R}^{h}$ where $ H : \mathcal{X} \to  \mathbb{R}^h $. For example, $H$ can be an image encoder from a multi-modal foundation model~\cite{radford2021learning}.

\subsection{Problem Definition}
\paragraph{Failure Detection} Also known as misclassification or error prediction~\cite{hendrycks17baseline}, failure detection aims to predict if a trained model makes an erroneous prediction on a test sample. In general, it requires a score for any given sample's prediction, where a lower score implies that the prediction is more likely to be wrong. 

For a standard network, the baseline method is to use maximum softmax output as the confidence score for failure detection ~\cite{hendrycks17baseline,Ovadia2019CanYT}:
\begin{align}
    \hat{p} \coloneqq \hat{P}(Y = \hat{y} \mid \x, B, w )
\end{align}
However, merely relying on the obtained confidence score from a newly trained classifier can be unsafe due to the overconfidence issue ~\cite{hendrycks17baseline,guo2017calibration} and this concern is even more pronounced under distribution shifts~\cite{Ovadia2019CanYT,hendrycks2019robustness,taori2020measuring}. We thus propose to employ information from foundation models to improve model reliability, instead of only using the information from the trained model.

\subsection{Inter-Model Latent Agreement Framework}

\paragraph{Overview} We propose an inter-model latent agreement framework to compute the agreement scores based on latent spaces and use it as an auxiliary source of information to boost the failure detection performance. The framework involves two steps: 1) measuring agreement between a newly trained model and a foundation model on samples; 2) fusing the agreement information into the predictive confidence in a post-hoc way via input-dependent temperature scaling.

\paragraph{Measuring Inter-Model Latent Agreement}

Neural network models usually first project input data into latent space, then perform classification or generation based on the latent spaces. Thus, latent spaces are informative and can be taken as the network's capacity for capturing the information in input samples. In this work, we aim to measure latent space agreement between models to represent their agreement.

Given an encoder $B$ from a trained classifier and a pretrained encoder $H$ from a foundation model, to estimate how reliable the trained classifier is on a sample $\x$, we aim to estimate the agreement between the models' latent spaces around $\x$. Specifically, we compute feature vectors $\z = B(\x)$ and $\h = H(\x)$ with the encoder $B$ and pretrained encoder $H$, respectively. However, as these latent spaces can be incoherent, \eg, differing by an unknown rotation, or different dimensions of latent spaces ($d \neq h$), this makes explicit distance comparison between $\z \in \mathbb{R}^d$ and $\h \in \mathbb{R}^h$ unsuitable.

To overcome this incompatibility, we compare \emph{neighborhoods} (e.g. nearest neighbors and distances to them) between two latent spaces around a sample as a surrogate task to measure the agreement instead of a direct distance measure between $\z$ and $\h$. For example, if two latent spaces are identical after a rotation transformation, the neighborhoods of a sample in the two latent spaces must be the same. Conversely, if the neighborhoods of a sample in the two latent spaces are similar, the two latent spaces around this sample are expected to be similar through some unknown transformation, due to the high level of agreement between these two latent spaces.

Specifically, denote the test sample as $\x^{\rm test}$, the encoder $B$ from a classifier and the classifier's training dataset $D = \{\mathbf{x}^{(i)}, y^{(i)}\}_{i=1}^n$. We obtain the feature vectors $\z^{\rm test} \coloneqq B(\x^{\rm test})$ and $\z_i \coloneqq B(\x^{(i)})$ for $1 \le i \le n$. We denote:
\begin{align}
    \mathbb{Z} \coloneqq (\z_1, \z_2, \dots, \z_n).  \label{eq:z_vec}
\end{align}
Similarly, we repeat this process with the pretrained encoder $H$ to get the feature vectors $\h^{\rm test} \coloneqq H(\x^{\rm test})$ and $\h_i \coloneqq H(\x^{(i)})$ for $1 \le i \le n$. We denote:
\begin{align}
    \mathbb{H} \coloneqq (\h_1, \h_2, \dots, \h_n). 
    \label{eq:h_vec}
\end{align}
To represent the ranking of training samples based on their similarity to the test sample in the latent space, we use a permutation generation function $G$ to produce a permutation containing the indexes of feature vectors in the training feature set $\mathbb{Z}$, ordered from nearest to farthest distance from the test feature vector $\z^{\rm test}$:
\begin{align}
    \begin{split}
    G(\z^{\rm test}, \mathbb{Z}) \coloneqq &(\Pi_{(1)}, \Pi_{(2)}, \dots, \Pi_{(n)}) \\
    \mathrm{s.t.} \quad s(\z^{\rm test}, \z_{\Pi_{(1)}}) \geq & s(\z^{\rm test}, \z_{\Pi_{(2)}}) \geq s(\z^{\rm test}, \z_{\Pi_{(n)}}),
    \end{split}
\end{align}
where we use the cosine similarity function as $s$. Similarly, we get another permutation using $\h^{\rm test}$ and $\mathbb{H}$ for the same test sample $\x^{\rm test}$ based on the pretrained encoder $H$. We use $\Pi^* \coloneqq G(\z^{\rm test}, \mathbb{Z})$ and $\Pi^\prime \coloneqq G(\h^{\rm test}, \mathbb{H})$ to represent the permutations obtained from the encoder $B$ from the classifier and the pretrained encoder $H$, respectively.

Next, to measure the similarity between two permutations $\Pi^*$ and $\Pi^\prime$, we introduce Normalized Discounted Cumulative Gain (NDCG), a ranking quality measure.

\begin{definition}
(Normalized Discounted Cumulative Gain (NDCG)~\cite{jarvelin2002cumulated}). 
Given a ranking $\Pi^*$ and another ranking $\Pi^\prime$, let $r$ denote our importance scoring function, where $r(i)$ outputs the importance score of the $i$-th sample, and $ r({\Pi^*_{(1)}}) \geq r({\Pi^*_{(2)}}) \geq \dots \geq  r({\Pi^*_{(n)}}) $:
\begin{align}\label{def:ndcg}
    \mathsf{NDCG}(\Pi^*, \Pi^\prime, r) \coloneqq 
    \frac{ \sum_{i}^n \frac{ r({\Pi^\prime_{(i)}}) } { \log{(i+1)} }}
         {{\sum_{i}^n \frac{ r({\Pi^*_{(i)}}) } { \log{(i+1)} }} }.
\end{align}
\end{definition}
The NDCG values range from $0$ to $1$ after normalization. Intuitively, the NDCG metrics quantitatively evaluate the ranking quality of $\Pi^\prime$ compared to the ranking $\Pi^*$, considering the importance scoring function $r$ and ranking position penalty with logarithmic discounting function. Note that any importance scoring function $r$ which satisfies the requirement of producing decreasing values according to the perfect ranking is plausible. In particular, $r$ can be a function outputting $0$ and $1$ depending on whether the training sample is one of the $k$-nearest training samples or not:
\begin{align}
    r(i) \coloneqq \mathbbm{1}(\z_i \in \mathcal{N}_{\mathbb{Z}, k} (\z^{\rm test}) ).
\end{align}
It means that we treat the nearest $k$ samples as most important, and we can control the neighborhood size with $k$. \citealt{wang2013theoretical} shows that this choice of importance scoring function also provides certain consistency guarantees.

As such, given the encoder $B$ from the trained classifier and pretrained encoders $H_1, \dots, H_m$ from $m$ different foundation models, we formally define our inter-model latent agreement score as follows:

\begin{definition}\label{def:agreement_score}
    (Inter-model Latent Agreement Score).
    Given a test sample $\x^{\rm test}$, the training samples $\{\x^{(i)}\}_{i=1}^n$, the encoder $B$ from a trained classifier and pretrained encoders $\mathcal{H} \coloneqq \{H_1, \dots, H_m\}$ from $m$ foundation models, recall $\mathbb{Z} \coloneqq (\z_1, \z_2, \dots, \z_n)$ and let $\mathbb{H}^i = (H_i(\x^{(1)}), H_i(\x^{(2)}), \dots, H_i(\x^{(n)})) $, the inter-model latent agreement score is:
    \begin{align}\label{eq:agreement_score}
        \mathsf{AS}(\x^{\rm test}, B, \mathcal{H}) \coloneqq \frac{1}{m} \sum_i^m \mathsf{NDCG} (\Pi^*, \Pi^i, r),
    \end{align}
    where $\Pi^* \coloneqq G(\z^{\rm test}, \mathbb{Z})$ and $\Pi^i \coloneqq G(H_i(\x^{\rm test}), \mathbb{H}^{i})$.
\end{definition}
The definition indicates that given a test sample, we average the ranking agreement across different foundation models as the inter-model latent agreement score.

\subsection{Main Empirical Observation}
\begin{figure}[t]
    \centering
    \includegraphics[width=0.5\linewidth]{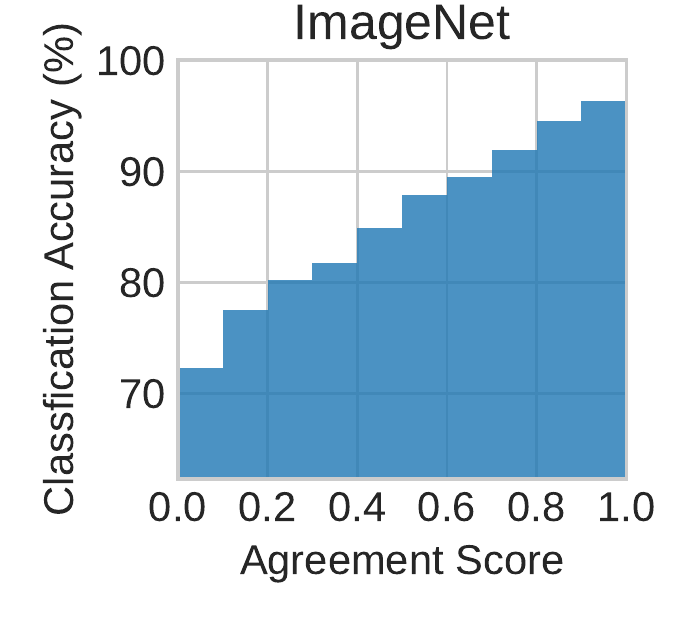}
    \caption{Positive correlation between agreement score and classification accuracy. The agreement score is calculated based on the ImageNet classifier (ViT-B/16) and CLIP ViT/L-14.}
    \label{fig:agree_vs_acc}
\end{figure}

From Figure \ref{fig:agree_vs_acc}, we observe that the agreement score has a clear positive correlation with the classification accuracy. This empirical evidence indicates that a prediction which has a higher latent space agreement with a foundation model tends to be predicted correctly in the classification task. This validates our use of agreement scores for assessing the reliability of a prediction and further improving the original predictive confidence to detect failure.

\subsection{Input-based Temperature Scaling}

As we aim to adjust the predictive confidence to improve failure detection performance but without altering the classifier's final predicted label, an appealling way is input-dependent temperature scaling, which is an extension of temperature scaling \cite{guo2017calibration,deng2022trust}. 
Classical temperature scaling uses a single scalar temperature parameter $t$ to rescale the softmax distribution. Using our agreement score for each sample $\x$ as prior information, we propose to obtain a scalar temperature $\tau(\x)$ as a learned function of the agreement score $\mathsf{AS}(\x, B, \mathcal{H})$ based on Definition \ref {def:agreement_score}: 
\begin{align}
    \tau(\x) & \coloneqq t + t_s \mathsf{AS}(\x, B, \mathcal{H}) \label{eq:sample_temp} \\
    \tilde{P}(Y \mid \x ) & \coloneqq \frac{\mathsf{softmax}(f_w(\z))}{ \tau(\x) } \label{eq:TS}
\end{align}
Here, $\tilde{P}(Y \mid \x)$ contains our output calibrated probabilities. $t$ and $t_s$ are learnable parameters; they are optimized via negative likelihood loss on the validation set, similarly to in classical temperature scaling~\cite{guo2017calibration}.
For each sample $\x$, we obtain $\tau(\x)$ as its input-dependent temperature. With $\tau(\x) = 1$, we recover the original predicted probabilities $\hat{p}$ for the sample. As all logit outputs of a sample are divided by the same scalar, the predictive label is unchanged. In this way, we calibrate the softmax distribution based on the agreement score, without compromising the model's accuracy. Note that though temperature scaling was mainly proposed for calibration, recent findings show temperature scaling can also improve failure detection ~\cite{galil2023what} by recalibrating the predictive probability distribution with proper scoring rules, which encourages both calibration and ranking for predictive confidence ~\cite{gneiting2007probabilistic,kuleshov2022calibrated}.

We summarize our framework in Algorithm \ref{alg:framework} in Appendix.

%% file: 040experiment.tex
In this section, we conduct experiments to answer the following research questions: 
\begin{itemize}[noitemsep,nolistsep]
    \item (Performance in ID: Section \ref{sec:perform_in_id}) How well does our method perform on in-distribution failure prediction compared to the baseline methods? 
    \item (Exploration Study: Section \ref{sec:exploration}) How do different foundation models affect our failure detection performance? How does this relate to the model family used?
    \item (Performance in OOD: Section \ref{sec:perform_in_ood}) How does it perform under OOD, \ie distribution shifts? 
    \item (Case Study: Section \ref{sec:case_study}) Can our method provide plausible explanation/visualization for samples with high/low agreement scores?
    \item (Ablation Study: Section \ref{sec:ablation}) How sensitive is the method to different hyperparameters? How does it perform when using other similarity measures?
    
\end{itemize}

\subsection{Experimental Setup}

\paragraph{Baselines}
Our baseline methods include the Maximum Softmax Probability (MSP)~\cite{hendrycks17baseline}, other uncertainty measure: Entropy and Energy~\cite{lakshminarayanan2017simple,liu2020energy}, a distance based measure: TrustScore~\cite{jiang2018trust}, MaxLogit proposed for distribution shift ~\cite{vaze2022openset} and vanilla Temperature Scaling (T.S.)~\cite{guo2017calibration}.

For details on datasets, classifiers, foundation/pretrained models, and experimental protocols, see Appendix \ref{appendix:exp_setup}.

\input{TAB/eval_on_id.tex}

\subsection{Failure Detection in ID data}
\label{sec:evaluation}
We first evaluate failure detection performance in in-distribution (ID) data and further explore the effect of using different pretrained models. We also study the correlation between each pretrained model's performance and its corresponding KNN accuracy in the same dataset. We also find strong correlation in performance within each model family.   
\paragraph{Performance Evaluation}\label{sec:perform_in_id}
The reported single-model result uses the model with the best ImageNet accuracy in our model candidate pool, CLIP ViT/L-14. For multiple-model settings, we adopt the models with top $2$ ImageNet accuracy: CLIP ViT/L-14 and ViT/L-16 (ImageNet-21K). We conduct further analysis about the effect of different pretrained models later. 
 
 As shown in Table \ref{tab:perform_on_id}, our method can outperform the baseline methods over different datasets, including large-scale dataset, ImageNet, for both CNN and ViT classifiers. Our empirical result confirms the previous findings that ViT can generally perform better than CNN classifiers in accuracy on uncertainty estimation-related tasks ~\cite{fort2021exploring,minderer2021revisiting,galil2022models}.

\paragraph{Effect of Different Foundation Models}
\label{sec:exploration}

Given the diversity of foundation models, which can vary in terms of architecture, training data, and optimization losses, it is important to further investigate the impact of different pretrained models on our proposed method.

We demonstrate the average performance over base models and datasets (excluding ImageNet\footnote{As some pretrained models are trained with ImageNet samples, we exclude ImageNet dataset.}) for different pretrained models in Figure \ref{fig:flops_vs_per}, which shows that every pretrained model in our model candidate pool can surpass MSP on average. We compare the performance of models pretrained on datasets of increasing size: ImageNet-1K, ImageNet-21K, and CLIP WebImageText dataset (400 million image-text pairs). Echoing the previous findings in transfer learning~\cite{kolesnikov2020big}, we also observe a performance boost of models of larger size  pretrained on larger datasets. With the similar model size and pretrained on the same dataset, ViTs generally perform better than CNNs, except for the cases where the dataset is in a smaller scale, \eg, ImageNet-1K. The self-supervised pretrained models, \eg MoCov3 ViT/ResNet, achieve comparable results with models pretrained with supervision signals. Similar to the previous findings in ~\cite{kornblith2019better}, which suggests that better performing models transfer better to the downstream tasks, we find the failure detection performance of pretrained models related to these models' accuracy performance in ImagetNet-1K, as shown in Figure \ref{fig:imagenet_acc_vs_perform_series}. Thus, we suggest to select the pretrained model with the highest accuracy on ImageNet-1K for failure detection, without any prior knowledge about the trained models, \eg the 
training datasets.

\begin{figure}[t]
    \centering
    \includegraphics[width=0.8\linewidth]{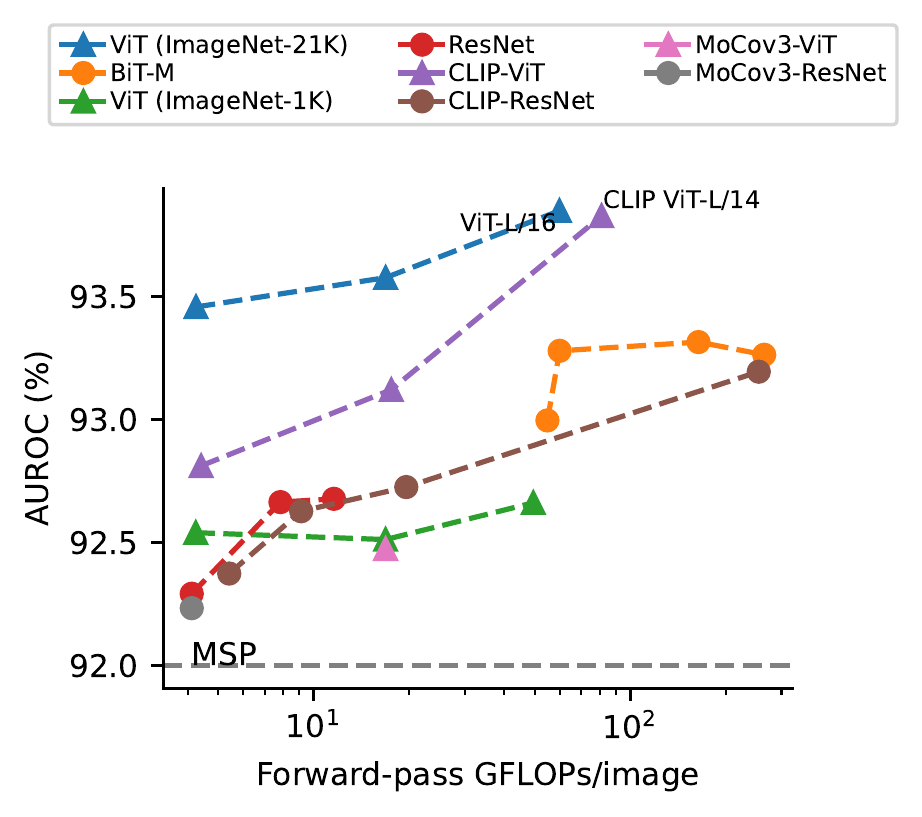}
    \caption{Performance of single-model with each pretrained model average over in-distribution datasets. x-axis: Inference GPU cost. See Appendix \ref{appendix:fail_detect} for plots for different base models.}
    \label{fig:flops_vs_per}
\end{figure}

\paragraph{Correlation Between Performance and KNN Accuracy}
\begin{figure}
    \centering
    \begin{subfigure}{0.49\linewidth}
        \includegraphics[width=\linewidth]{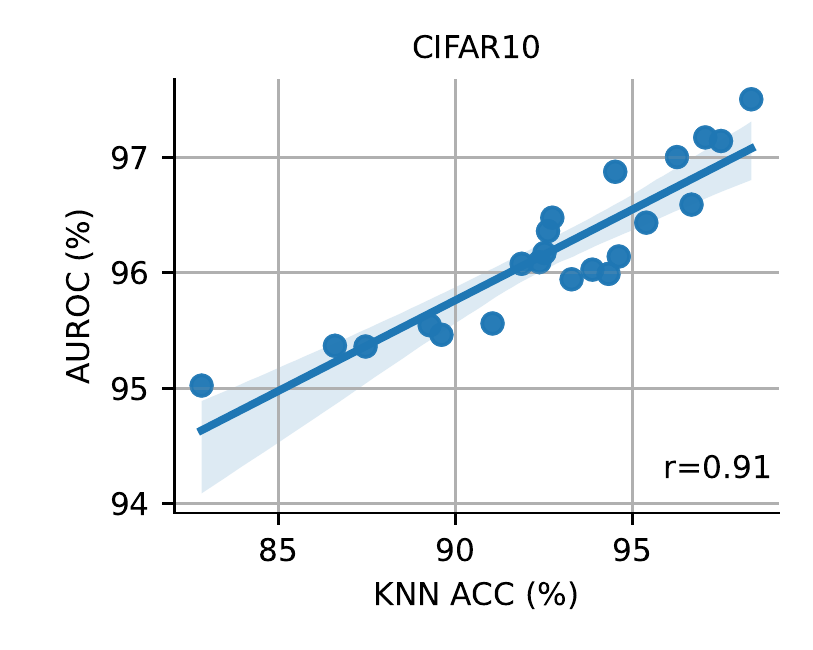}
    \end{subfigure}
    \begin{subfigure}{0.49\linewidth}
        \includegraphics[width=\linewidth]{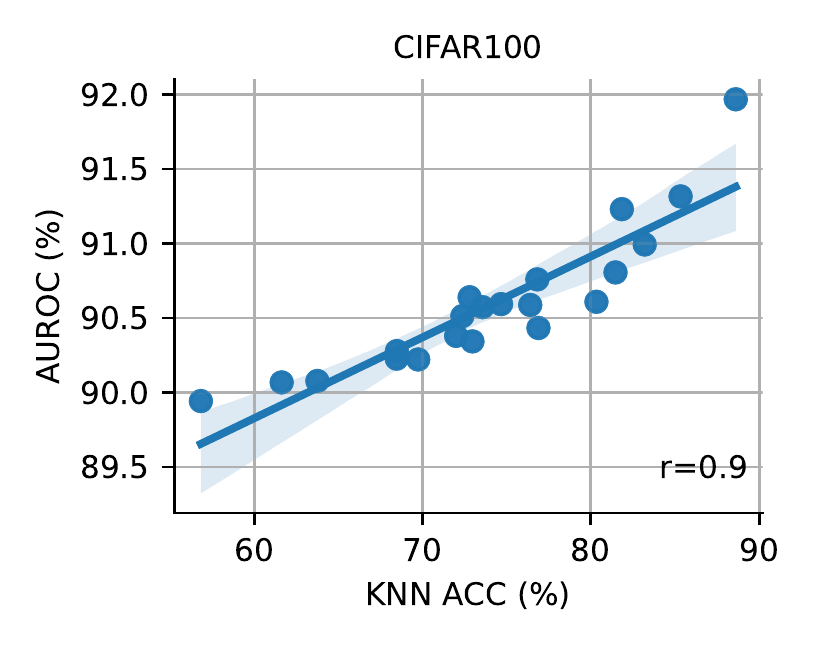}
    \end{subfigure}
    \caption{Strong correlation between failure detection performance and KNN accuracy for each pretrained model. See Figure \ref{tab:knn_acc_vs_per_all} for plots for other datasets.}
    \label{fig:knn_acc_vs_per}
\end{figure}

We further investigate the failure detection performance for each pretrained model in a particular dataset. Note that we use KNN classifier~\cite{wu2018unsupervised,caron2021emerging}, a simple weighted nearest neighbor classifier, as a performance proxy to evaluate the pretrained model's performance on the downstream task ~\cite{renggli2022model}. Figure \ref{fig:knn_acc_vs_per} and \ref{tab:knn_acc_vs_per_all} show the strong correlation between the failure detection performance and the KNN classifier performance for each pretrained model, which implies that a pretrained model with a potentially better performance in a downstream dataset can be more useful to detect failures for a model trained with this dataset.

\paragraph{Agreement Correlation Within and Between Model Families}
\begin{figure}[t]
    \centering
    \includegraphics[width=.8\linewidth]{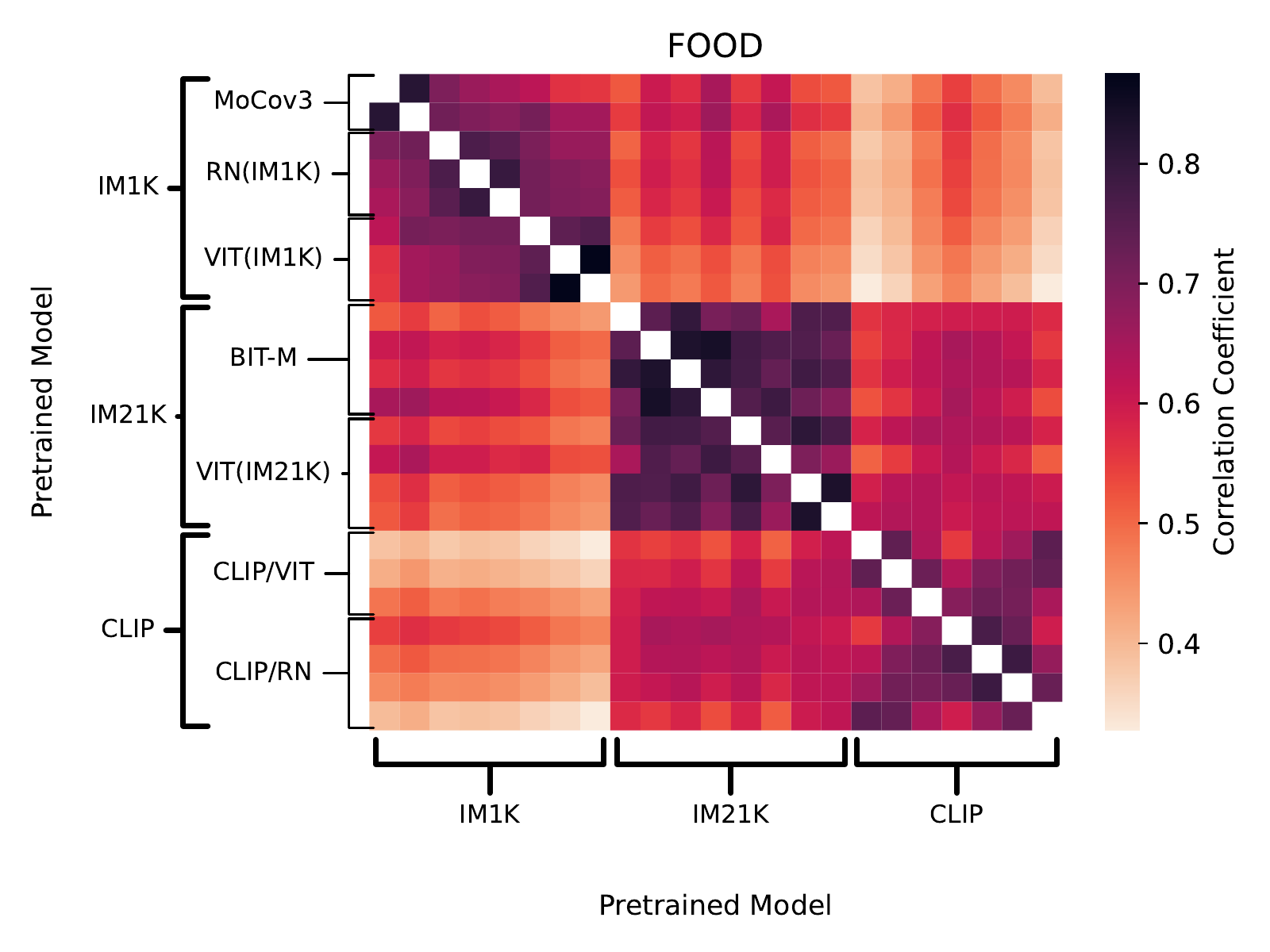}
    \caption{Pairwise correlation heatmap among any two pretrained models on FOOD dataset. Pretrained models trained with the same dataset tend to give similar information about neighborhood agreement, regardless of the architecture. IM1K: ImageNet-1K, IM21K: ImageNet-21K, indicating the pretraining dataset of the pretrained model. See Figure \ref{fig:cor_heatmap_all} for plots for each dataset.}
    \label{fig:cor_heatmap}
\end{figure}

Beyond studying the effect of each pretrained model, we further investigate when two pretrained models tend to have similar agreement on the same sample. Specifically, we compute neighbor agreement between a trained model and a pretrained model for each sample. Hence, we obtain the pairwise correlation, \ie Pearson Correlation, between any two pretrained models. Figure \ref{fig:cor_heatmap} and \ref{fig:cor_heatmap_all} display the overall correlation map and show an interesting observation: the models pretrained on the same datasets tend to give more similar information and are less affected by the architectures. This observation also implies that when pretrained models can achieve comparable performance, we should select the models that use different pre-training data to benefit from the more diverse information.

\subsection{Failure Detection in OOD}
\label{sec:perform_in_ood}
\input{TAB/eval_on_shift.tex}

We further verify our method under distribution shifts, which simulate the challenging tasks encountered in a real-world deployment. Classifiers may produce more inaccurate predictions when dealing with unseen data, highlighting the need for failure detection for safety.

The results are shown in Table \ref{tab:perform_on_shift}, where we treat CIFAR10 and CIFAR100 as ID data and train classifiers based on these data, and test them on data distributions unseen during training: CIFAR10.1, corrupted CIFAR10 and CIFAR100. Similarly, we test on distribution shifts, ImageNetV2 and ImageNet-Sketch for ImageNet as ID data. All evaluation settings follow those used in ID data evaluation. The results show that our method can generalize well in distribution shift cases and generally outperform the baselines.

\subsection{Case Study}
\label{sec:case_study}
\begin{figure}
    \centering
    \includegraphics[width=\linewidth]{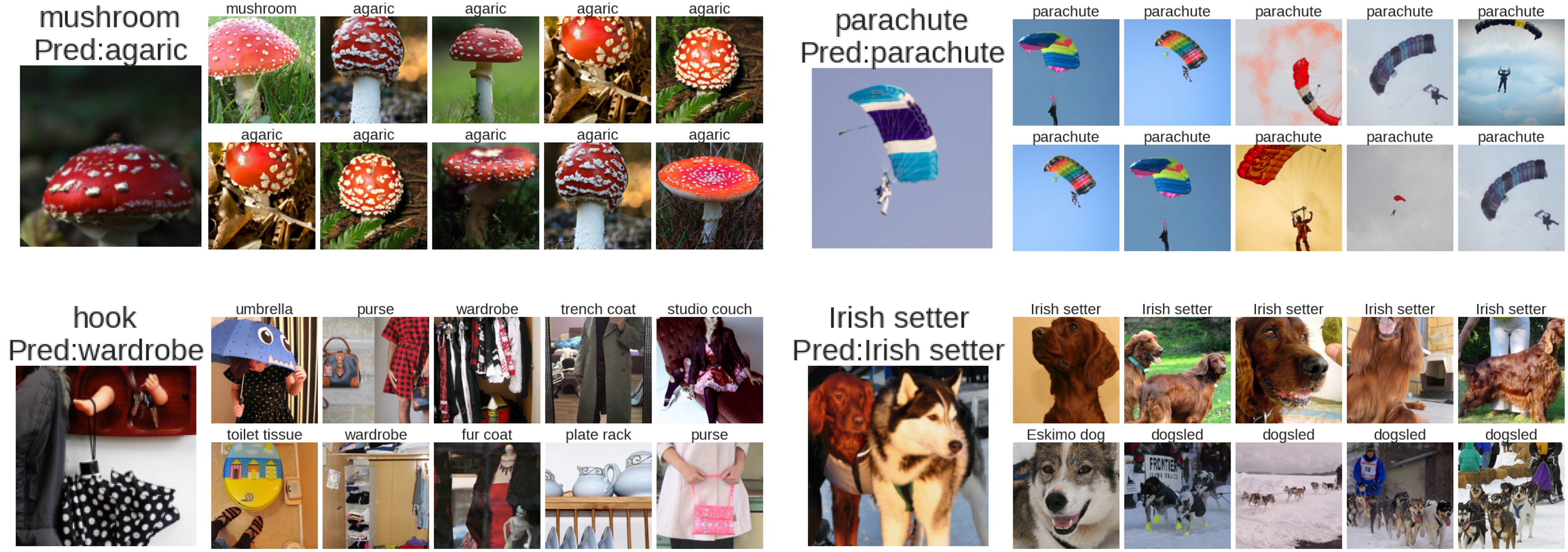}
    \caption{Four exemplar samples in ImageNet.top/bottom: samples with high/low agreement scores; left/right: failure/correct predictions. The top and bottom row of each sample box shows neighborhood samples under the trained model(fine-tuned ViT/B-16) and foundation model (\ie, CLIP ViT/L-14), respectively.  The samples with lower agreement scores tend to have multiple objects in the pictures, leading to intrinsic difficulties for the model in correctly predicting and thus disagreement between models.}
    \label{fig:case_study}
\end{figure}
\paragraph{What samples tend to get lower/higher agreement scores?}
Besides the numeric results, we also take a closer look at the samples that receive lower or higher agreement scores. Notably, as our method is based on neighborhood similarity, it enables some explanation ability by inspecting the difference between the neighborhood samples obtained from different models, especially when low agreement scores occur. We would like to highlight this as it enables human experts to further investigate the root cause of failed predictions.

Figure \ref{fig:case_study} shows the erroneous and correct samples with high and low agreement scores in ImageNet. We visualize the nearest $5$ samples under the trained classifier (fine-tuned ViT/B-16) and a foundation model (CLIP ViT/L-14). We observe that samples with low agreement scores tend to be complex or contain multiple objects, leading to intrinsic difficulties for the model in correctly predicting and thus disagreement between models; while samples with high agreement scores usually contain a single prominent object.

\subsection{Ablation Study}
\label{sec:ablation}
\paragraph{Effect of $k$ and neighbor candidate pool size $n$}

In Figure \ref{fig:ablation_KN}, we analyze the effect of k and the neighbor candidate pool size $n$ on two different datasets: CIFAR10 and CIFAR100 datasets, by varying the number of neighbors $k \in \{ 10, 20, 50, 100, 200, 500, 1000 \}$ and neighbor candidate pool size $n \in \{ 2000,5000,10000,20000,50000 \}$. It shows that the average performance is better and less sensitive to the choice $k$ with increasing pool size $n$. In Table \ref{tab:ablation_hyper}, we can see that even using a small subset of training data (e.g. $n=2000$, $<5\%$ training data from CIFAR10/CIFAR100) can already perform better than the baseline methods. Generally, the performance improves with larger pool sizes but stabilizes around an N of 10000 to 20000 for CIFAR10/CIFAR100.

\paragraph{Other choices of agreement measure}

To further study the impact of agreement measure, we replace NDCG with other similarity measures: Spearman's rank correlation coefficient, Centered Kernel Alignment (CKA)~\cite{kornblith2019similarity} and Jaccard Similarity between $k$-hop neighborhood sets. For fair comparison, we also use $k$-hop neighborhood samples to compute CKA. The linear CKA and RBF-kernel CKA report similar results in our experiments. As shown in Table ~\ref{tab:ablation_score}, the results imply the importance of neighborhoods, adaptivity to more general transformations, and ranking information, by the comparison with Spearman, CKA, and Jaccard respectively. Note that, the adaptivity to more general transformations is most important among these factors as NDCG and Jaccard outperform CKA and Spearman, which might be credited to the fact that NDCG is invariant to more general transformations compared to CKA.

\begin{figure}[]
    \centering
    \begin{subfigure}[b]{.49\linewidth}
        \centering
        \includegraphics[width=\linewidth]{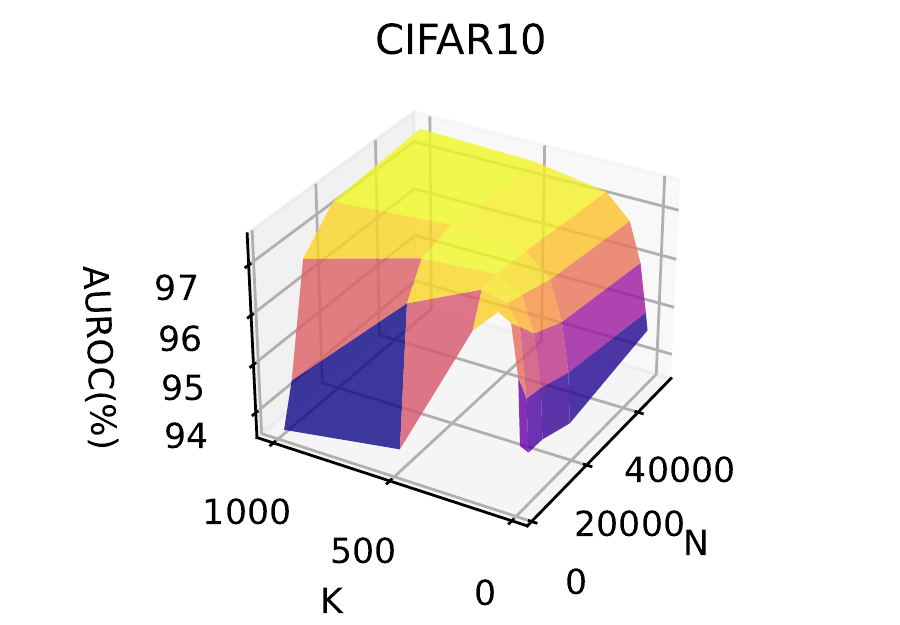}
        \label{fig:ablation_KN_cifar10}
    \end{subfigure}
    \begin{subfigure}[b]{0.49\linewidth}
        \centering
        \includegraphics[width=\linewidth]{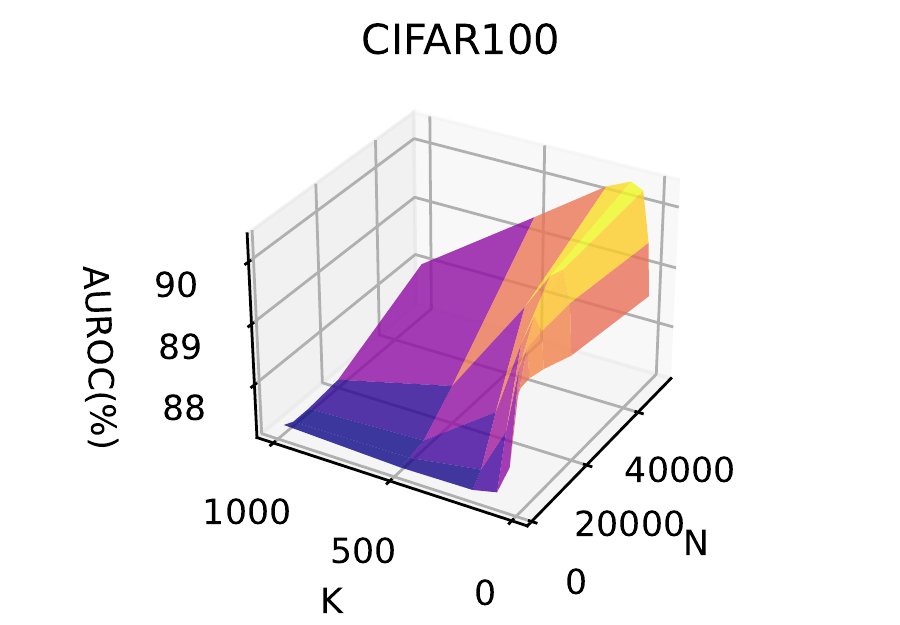}
        \label{fig:ablation_KN_cifar100}
    \end{subfigure}
    \caption{Ablation study with different $k$ and $n$ for CIFAR10 and CIFAR100. See the numeric results in Table \ref{tab:ablation_hyper}.}
    \label{fig:ablation_KN}
\end{figure}

%% file: TAB/eval_on_id.tex
\begin{table*}[h]
\centering
\caption{AUROC (\%) averaged over $6$ runs. The base models are finetuned based on pretrained models: ResNet-50 or ViT-B/16. ViT for ImagNet is fine-tuned on CLIP VIT-B/16 model. $\mathsf{AS}$: inter-model latent agreement score in Eq\eqref{eq:agreement_score}. single: uses CLIP ViT/L-14 as foundation model. multiple: uses CLIP ViT/L-14 and ViT/L-16 (ImageNet-21K) as foundation models. The best result is bolded.}
\label{tab:perform_on_id}
    \centering
    \scalebox{.9}{
    \begin{tabular}{@{}llcccccc@{}}
    \toprule
                  &  & CIFAR10       & CIFAR100      & STL         & BIRDS         & FOOD   & ImageNet        \\ 
    Model          & Method                &               &               &                &                &      &  \\ \cmidrule(l){1-8} 
     \multirow{7}{*}{CNN} 
 & MSP                              & 94.15$_{\pm 0.06} $          & 88.42$_{\pm 0.21} $          & 95.73$_{\pm 0.35} $          & 85.97$_{\pm 0.76} $          & 88.79$_{\pm 0.19} $          & 86.19$_{\pm 0.07} $          \\
 & Entropy                          & 94.14$_{\pm 0.07} $          & 88.45$_{\pm 0.22} $          & 95.62$_{\pm 0.36} $          & 84.83$_{\pm 0.88} $          & 88.78$_{\pm 0.21} $          & 84.05$_{\pm 0.06} $          \\
 & Energy                           & 90.83$_{\pm 0.32} $          & 83.84$_{\pm 0.29} $          & 94.31$_{\pm 0.64} $          & 77.72$_{\pm 2.10} $          & 83.83$_{\pm 0.19} $          & 72.72$_{\pm 0.11} $          \\

 & MaxLogit                         & 91.00$_{\pm 0.32} $          & 84.20$_{\pm 0.30} $          & 94.44$_{\pm 0.63} $          & 79.82$_{\pm 1.76} $          & 84.59$_{\pm 0.20} $          & 76.49$_{\pm 0.11} $          \\

& TrustScore & 95.84$_{\pm 0.24} $ & 89.13$_{\pm 0.20} $ & 97.70$_{\pm 0.20} $ & 84.45$_{\pm 1.27} $          & 85.71$_{\pm 0.26} $          & 75.44$_{\pm 1.19} $          \\

 & T.S.                     & 93.76$_{\pm 0.14} $          & 87.21$_{\pm 0.22} $          & 95.54$_{\pm 0.37} $          & 85.88$_{\pm 0.79} $          & 88.58$_{\pm 0.19} $          & 86.35$_{\pm 0.07} $          \\
 \cmidrule{2-8}
 & T.S. (w/ $\mathsf{AS}$, single)   & 97.45$_{\pm 0.05} $          & 89.63$_{\pm 0.22} $          & 99.32$_{\pm 0.12} $          & 88.27$_{\pm 0.71} $          & \textbf{92.33}$_{\pm 0.14} $ & 86.38$_{\pm 0.08} $         \\
 & T.S. (w/ $\mathsf{AS}$, multiple) & \textbf{98.07}$_{\pm 0.08} $ & \textbf{91.04}$_{\pm 0.24} $ & \textbf{99.33}$_{\pm 0.13} $ & \textbf{89.10}$_{\pm 0.65} $ & 92.17$_{\pm 0.16} $          & \textbf{86.39}$_{\pm 0.09} $ \\
    
    \cmidrule(l){1-8} 
\\
 \multirow{7}{*}{ViT}  
& MSP                              & 96.39$_{\pm 0.44} $          & 92.66$_{\pm 0.16} $          & 98.64$_{\pm 0.40} $          & 88.23$_{\pm 0.45} $          & 92.77$_{\pm 0.26} $          & 85.42$_{\pm 0.09} $          \\
 & Entropy                          & 96.36$_{\pm 0.44} $          & 92.52$_{\pm 0.17} $          & 98.60$_{\pm 0.40} $          & 87.58$_{\pm 0.45} $          & 92.83$_{\pm 0.27} $          & 81.92$_{\pm 0.09} $          \\
 & Energy                           & 91.56$_{\pm 0.76} $          & 83.67$_{\pm 0.52} $          & 93.52$_{\pm 1.24} $          & 76.73$_{\pm 0.72} $          & 87.00$_{\pm 0.32} $          & 65.81$_{\pm 0.11} $          \\
 & MaxLogit                         & 91.71$_{\pm 0.73} $          & 84.49$_{\pm 0.45} $          & 94.15$_{\pm 1.21} $          & 79.31$_{\pm 0.72} $          & 87.67$_{\pm 0.33} $          & 74.48$_{\pm 0.13} $          \\
 & TrustScore & 97.14$_{\pm 0.31} $ & 92.33$_{\pm 0.13} $          & 99.32$_{\pm 0.20} $ & 87.49$_{\pm 0.61} $          & 89.03$_{\pm 0.46} $          & 82.49$_{\pm 0.58} $          \\
 & T.S.                     & 96.11$_{\pm 0.50} $          & 92.31$_{\pm 0.12} $          & 98.63$_{\pm 0.40} $          & 88.11$_{\pm 0.46} $          & 92.83$_{\pm 0.25} $          & 86.44$_{\pm 0.09} $          \\
 \cmidrule{2-8}
  & T.S. (w/ $\mathsf{AS}$, single)   & 96.84$_{\pm 0.31} $          & 92.83$_{\pm 0.17} $          & \textbf{99.46}$_{\pm 0.18} $ & 88.78$_{\pm 0.47} $          & \textbf{93.81}$_{\pm 0.24} $ & 86.97$_{\pm 0.10} $            \\
 & T.S. (w/ $\mathsf{AS}$, multiple) & \textbf{97.36}$_{\pm 0.25} $ & \textbf{93.14}$_{\pm 0.15} $ & 99.34$_{\pm 0.28} $          & \textbf{88.98}$_{\pm 0.47} $ & 93.64$_{\pm 0.24} $          & \textbf{87.36}$_{\pm 0.10} $ \\
 \bottomrule
\end{tabular}}

\end{table*}

%% file: TAB/eval_on_shift.tex
\begin{table*}[!ht]
\centering
\caption{AUROC (\%) Performance on distribution shift datasets averaged on $6$ runs. All evaluation settings follow those used in Table \ref{tab:perform_on_id}.}
\label{tab:perform_on_shift}
    \centering
    \scalebox{.9}{
    \begin{tabular}{llcccccc}
    \toprule
                  &  & CIFAR10.1       & CIFAR10-C      & CIFAR100-C         & ImageNetV2         & ImageNet-SK           \\ 
    Model          & Method                &               &               &                &                  \\ \cmidrule(l){1-7} 
     \multirow{7}{*}{CNN} 
 & MSP                              & 89.09$_{\pm 1.08} $          & 77.83$_{\pm 1.46} $          & 77.70$_{\pm 0.66} $          & 83.93$_{\pm 0.00} $          & 79.40$_{\pm 0.00} $          \\
  & Entropy                          & 89.11$_{\pm 1.12} $          & 78.03$_{\pm 1.53} $          & 78.41$_{\pm 0.71} $          & 81.63$_{\pm 0.00} $          & \textbf{80.01}$_{\pm 0.00} $ \\
 & Energy                           & 86.94$_{\pm 1.37} $          & 77.20$_{\pm 1.84} $          & 79.18$_{\pm 0.72} $          & 71.24$_{\pm 0.00} $          & 75.77$_{\pm 0.00} $          \\

 & MaxLogit                         & 87.09$_{\pm 1.35} $          & 77.40$_{\pm 1.83} $          & 79.38$_{\pm 0.72} $          & 75.37$_{\pm 0.00} $          & 77.56$_{\pm 0.00} $          \\

 & TrustScore &  91.39$_{\pm 0.19} $ &  82.01$_{\pm 0.77} $ &  80.81$_{\pm 0.57} $ &           74.91$_{\pm 1.01} $ &           74.04$_{\pm 1.26} $ \\
 & T.S.                     & 88.83$_{\pm 1.18} $          & 78.14$_{\pm 1.58} $          & 78.57$_{\pm 0.81} $          & 84.07$_{\pm 0.00} $          & 79.05$_{\pm 0.02} $          \\
 \cmidrule{2-7}

 & T.S. (w/ $\mathsf{AS}$, single)   & 95.20$_{\pm 0.71} $          & 88.96$_{\pm 1.20} $          & 80.64$_{\pm 0.78} $          & 84.11$_{\pm 0.03} $          & 79.16$_{\pm 0.06} $          \\
  & T.S. (w/ $\mathsf{AS}$, multiple) & \textbf{96.09}$_{\pm 0.65} $ & \textbf{92.76}$_{\pm 0.53} $ & \textbf{83.54}$_{\pm 0.60} $ & \textbf{84.11}$_{\pm 0.04} $ & 79.16$_{\pm 0.10} $          \\
    
    \cmidrule(l){1-7} \\

\multirow{7}{*}{ViT} 
 & MSP                              & 96.46$_{\pm 0.27} $          & 92.42$_{\pm 0.62} $          & 87.61$_{\pm 0.53} $          & 82.87$_{\pm 0.00} $          & 81.81$_{\pm 0.00} $          \\
& Entropy                          & 96.41$_{\pm 0.28} $          & 92.42$_{\pm 0.62} $          & 87.75$_{\pm 0.49} $          & 80.08$_{\pm 0.00} $          & 80.24$_{\pm 0.00} $          \\
 & Energy                           & 91.39$_{\pm 0.47} $          & 89.19$_{\pm 1.24} $          & 83.39$_{\pm 0.33} $          & 66.39$_{\pm 0.00} $          & 72.43$_{\pm 0.00} $          \\
 & MaxLogit                         & 91.69$_{\pm 0.43} $          & 89.42$_{\pm 1.23} $          & 83.96$_{\pm 0.35} $          & 73.61$_{\pm 0.00} $          & 77.05$_{\pm 0.00} $          \\

 & TrustScore &           96.08$_{\pm 0.36} $ &           92.21$_{\pm 0.77} $ &  88.00$_{\pm 0.47} $ &           81.71$_{\pm 0.43} $ &           80.31$_{\pm 0.31} $ \\
 & T.S.                     & 96.23$_{\pm 0.29} $          & 92.35$_{\pm 0.66} $          & 87.69$_{\pm 0.48} $          & 83.68$_{\pm 0.02} $          & 82.00$_{\pm 0.00} $          \\
 \cmidrule{2-7}

 & T.S. (w/ $\mathsf{AS}$, single)   & 96.28$_{\pm 0.52} $          & 92.61$_{\pm 0.52} $          & 87.87$_{\pm 0.49} $          & 84.38$_{\pm 0.07} $          & 83.00$_{\pm 0.10} $         \\
 & T.S. (w/ $\mathsf{AS}$, multiple) & \textbf{96.64}$_{\pm 0.39} $ & \textbf{93.48}$_{\pm 0.33} $ & \textbf{88.35}$_{\pm 0.48} $ & \textbf{84.77}$_{\pm 0.06} $ & \textbf{83.29}$_{\pm 0.10} $ \\
\bottomrule
    \end{tabular}}

\end{table*}

%% file: theory.tex
In this section, we discuss how foundation models correlate with the reliability of a trained model prediction by latent spaces agreement and justify the use of NDCG scores as latent spaces agreement.

 \paragraph{Setup} For analysis, we discuss the cases under a regression problem. Let the training dataset contains $N$ samples, $D^\prime = \{\x^{(i)}, y^{(i)}\}_{i=1}^N$, where $ \mathbf{x}^{(i)} \in  \mathbb{R}^d$ is the $i$-th input sample and $y^{(i)}$ is the ground-truth scalar value. We can denote the predictor $f_{w,B}(x) = \mathbf{w}^\top B(\x) $, which consists of a feature extractor with normalization $B : \mathbb{R}^d \to \mathbb{R}^k$, and a weight vector $\mathbf{w} \in \mathbb{R}^{k \times 1}$. Thus, given the training dataset $D^\prime$, one can obtain a well-trained predictor $f_{w_0, B_0}(x)$ by minimizing a loss function $l$, \ie $\mathbf{w}_0, B_0 = \argmin_{\mathbf{w}, B} l(\x, y, \mathbf{w}, B)$, where the loss function can be squared loss, \etc. We use $\Vert \cdot \Vert$ as $\ell_2$ norm. Let $H: \mathbb{R}^d \to \mathbb{R}^k$ be the pretrained encoder for a foundation model. 

\subsection{Relation between Prediction Error and Latent Space Agreement}
 We denote loss function $l(\x, y, B, \w) \coloneqq \Vert \w^\top B(\x) - y \Vert $. Let $\w_0 \coloneqq \argmin_{\w} \frac{1}{n} \sum_{x}{ \left\Vert \w^{\top}B_0(x) - y \right\Vert } $. We assume there exists an isometric transformation $U_h \in \mathcal{U} $, where  $\mathcal{U}$ contains all possible isometric transformations and $U_h \coloneqq \argmin_{U} \mathbb{E} \left\Vert B_0(X) - U H(X)  \right\Vert $. We assume that there exists a head $\w_h \in \mathbb{R}^{k \times 1}$, where $ \left\Vert \w_h^\top U_h^{-1} - \w_0^\top \right\Vert = \left\Vert \Delta \right\Vert \leq C $ . We use normalized features, \ie $\left\Vert B_0(\x) \right\Vert = 1$.
\begin{proposition}
Given a test sample $\x$ and its ground-truth value $y$, the trained encoder $B_0$ and its linear head $\w_0$, we have a pretrained encoder $H$ from a foundation model. If the pretrained encoder predicts correctly: $ l(\x, y, H, \w_h) = 0 $, the prediction error $l(\x, y, B_0, \w_0) \leq  ( C + \left\Vert \w_0 \right\Vert) \cdot \left\Vert B_0(\x) - U_h H(\x) \right\Vert + C$.
\end{proposition}
\begin{proof}%
The detailed proof is relegated to Appendix \ref{appendix:proof_test_error}.
\end{proof}

In summary, that is if two latent spaces highly agree on a sample $\x$, \ie $\left\Vert B_0(\x) - U_h H(\x) \right\Vert$ close to $0$, the model is more likely to predict accurately on $\x$. However, measuring this latent space agreement $\left\Vert B_0(\x) - U_h H(\x) \right\Vert$ can be challenging as there exists an unknown isometric transformation $U_h$.

\subsection{ Local Approximation Isometry and NDCG }
Next, we show that NDCG provides an effective measure of similarity between latent spaces irrespective of rotations or distortion. To show this, we first define a transformation $f$ which approximately preserves distances around $x$, as a  $\delta$-Local Approximation Isometry:
\begin{assumption}
($\delta$-Local Approximation Isometry). $\forall \z \in \mathcal{N}_k(\x), \exists \delta \geq 1, \frac{ \left\Vert f(\z) - f(\x) \right\Vert }{ \left\Vert \z - \x \right\Vert } \in \left( \frac{1}{\delta}, \delta \right)$.
\end{assumption}
Intuitively, $\delta$ describes how `approximately' the distances are preserved by $f$. For example, when $\delta = 1$, all the samples around $x$ are strictly distance-preserving, so the neighborhood and ranking of neighborhood samples do not change. As $\delta$ increases, the neighborhood after applying $f$ differs more from the original neighborhood, as does the ranking of neighborhood samples. 
\begin{proposition}
(Lower Bound of NDCG scores).
Given an input sample $\x$, $\Pi^*$ and $\Pi^\prime$ are permutations before and after a $\delta$-local approximation isometric transformation $f$, we have $\mathsf{NDCG}( \Pi^*, \Pi^\prime, r ) \geq \frac{1}{\delta^2}$, when $r = 1/d(\cdot, \x)$ and $d$ is a distance scoring function.
\end{proposition}
\begin{proof}
    The detailed proof is relegated to Appendix \ref{appendix:proof_ndcg_bound}.
\end{proof}

This shows that if a local approximate isometry exists near a point $\x$, the NDCG is guaranteed to be high $(\ge 1/\delta^2)$. As $\delta$ approaches $1$, the NDCG also approaches $1$.

%% file: 050conclusion.tex
While powerful foundation models have received increasing attention, their use in improving model reliability is still underexplored due to the challenges of incompatible latent spaces between foundation models and a trained classifier. In this paper, we proposed a novel inter-model latent agreement framework to overcome this incompatible issue and improve the reliability of a trained classifier without any fine-tuning. We first show the agreement correlates well with classification accuracy. Motivated by this, our framework enables incorporating the agreement score into predictive confidence to improve failure detection performance. We conduct extensive experiments on failure detection to verify the benefits of our framework to improve model reliability and provide theoretical justification for our method. We believe our proposed neighborhood agreement measure between latent spaces can further benefit the study of the interconnection between different models.

%% file: 060appendix.tex
\appendix
\section{Experimental Setup}
\label{appendix:exp_setup}

\subsection{Datasets}

We run experiments on six in-distribution datasets and five distribution shifts to evaluate the failure detection performance. For in-distribution, we use CIFAR10~\cite{cifar10}, CIFAR100, STL~\cite{coates2011analysis}, BIRDS~\cite{WahCUB_200_2011}, FOOD~\cite{bossard14} and a large-scale dataset, ImageNet (ImageNet-1K)~\cite{deng2009imagenet}. For distribution shifts, we use CIFAR10.1~\cite{recht2018cifar10.1}, Gaussian Blur Corrupted CIFAR10 samples (CIFAR10-C)~\cite{hendrycks2019robustness} as natural and corruption distribution shift for CIFAR10. We use Gaussian Blur Corrupted CIFAR100 samples (CIFAR100-C) as corruption distribution shift for CIFAR100. For ImageNet, we fine-tune on ImageNet and evaluate failure detection on distribution shifts: ImageNetV2~\cite{recht2019imagenet} and ImageNet-Sketch (ImageNet-SK)~\cite{wang2019learning}. See details about datasets and split settings in Table \ref{appendix:dataset_table}.

\begin{table*}[h]
    \centering
    \caption{Number of images per data set and associated splits}
    \label{appendix:dataset_table}
    \begin{tabular}{lrrrrr}
    \toprule
        Datasets & Classes & Train Size & Val. Size & Test Size & Unlabeled Set Size \\
            \midrule
         CIFAR10 & 10 & 50000 & 1000 & 9000 & - \\
         CIFAR100 & 100 & 50000 & 1000 & 9000 & - \\
         BIRDS & 200 &  5994 & 2897 & 2897 & - \\
         STL & 10 & 5000 & 4000 & 4000 & 100000 \\
         FOOD & 102 & 75750 & 12625 & 12625 & - \\
         ImageNet & 1000 & 1281167 & 10000 & 40000 & - \\
    \midrule
        CIFAR10-C & 10 & - & - & 10000 & - \\
        CIFAR100-C & 10 & - & - & 10000 & - \\
        CIFAR10.1 & 10 & - & - & 2000 & - \\
        ImageNetV2 & 1000 & - & - & 10000 & - \\
        ImageNet-Sketch & 1000 & - & - & 50000 & - \\
    \bottomrule
    \end{tabular}
\end{table*}

\subsection{Base Models}

We consider two common architectures: CNN-base (ResNet-50) and ViT-base (ViT-B/16) models as base classifiers across all datasets. To see if our method can still outperform on "pretrained and fine-tuned" models, all our base classifiers are initialized with a larger-scale data pretrained model and then fine-tuned. For all datasets except ImageNet, our base classifiers are trained with initializing with ResNet-50 architecture pretrained on ImageNet-1K examples and ViT/B-16 model pretrained on ImageNet-21K examples. For ImageNet, we use public fine-tuned models from PyTorch Image Models~\cite{rw2019timm}, which are ImageNet-21K pretrained ResNetV2-50 model and CLIP pretrained ViT/B-16 model. We use penultimate layer output as the encoding feature vectors for the trained models.

\paragraph{Model Architectures and pretraining source}
For all datasets except for ImageNet, our base models are trained with initializing with ResNet-50 model pretrained on ImageNet-1K examples and ViT/B-16 model pretrained on ImageNet-21K examples. For ImageNet, we use public fine-tuned models from TIMM ~\cite{rw2019timm}, which are fine-tuned on ImageNet based on CLIP ViT/B-16 model.

\paragraph{Training Receipt}
For ResNet-50 models, we fine-tune with Adam optimizer with learning rate $1e-4$ and $(\beta_1, \beta_2) = (0.9, 0.99)$. For CIFAR10, CIFAR100, STL and BIRDS, we fine-tune for $50$ epochs. For FOOD, we fine-tune for $20$ epochs. We use the public trained ImageNet classifier from ~\cite{rw2019timm}. 

For ViT, we fine-tuned with cosine annealing scheduler. The detail is shown in Table ~\ref{appendix:vit_train}.
\begin{table*}[h]
    \centering
    \caption{Training parameters per data set for ViT. init-lr: Initial learning rate of the cosine annealing scheduler as selected. steps: Number of batches that was trained on. }
    \label{appendix:vit_train}
    \begin{tabular}{lrrrr}
    \toprule
        Datasets & init-lr  & batch size & steps \\
            \midrule
         CIFAR10 & 3e-4  & 64 & 15000 \\
         CIFAR100 & 3e-4  & 64 & 15000 \\
         BIRDS & 3e-4 &  64 & 5000 \\
         STL & 3e-4 &  64 & 4000 \\
         FOOD & 3e-4 & 32 & 47000 \\
    \bottomrule
    \end{tabular}
\end{table*}

\paragraph{Base Models Performance}
\begin{table*}
    \centering
    \caption{Average classification accuracy (\%) of trained classifiers used in our paper.}
    \label{tab:classifer_acc}
    \begin{tabular}{c|cccccc}
        \toprule
         Model &  CIFAR10 & CIFAR100 & STL & BIRDS & FOOD & ImageNet  \\
         \midrule
         CNN &  97.36 &  85.04 & 97.79 & 77.68 & 81.53 & 80.31 \\
         ViT &  99.09 & 93.47 &  99.33 &  84.76 &  90.60 &  85.24 \\
         \bottomrule
    \end{tabular}
\end{table*}

For sanity check of trained classifiers, we show the average classification accuracy of our trained models used in this paper in Table ~\ref{tab:classifer_acc}.

\subsection{Foundation/Pretrained Models}

In this work, we have included $23$ public pretrained models of diverse architectures, training data and optimization losses. For specific, these models can be categorized into $5$ model families as: CLIP ViT/ResNet~\cite{radford2021learning}, ViT (pretrained on ImageNet-21K or ImageNet-1K)~\cite{dosovitskiy2020image,steinerhow}, BiT-M~\cite{kolesnikov2020big}, ResNet~\cite{he2016deep} and MoCov3~\cite{Chen_2021_ICCV}. Among the candidate models, the model with best fine-tune ImageNet accuracy is CLIP ViT/L-14 (87.85\%) and the second best is ViT/L-16 (ImageNet-21K) (87.08\%)\footnote{as reported in ~\cite{rw2019timm}.}. We use penultimate layer output as the encoding feature vectors for the pretrained models. Except for multi-modal foundation model CLIP, we use the image encoders. We introduce the model families as follows:
\begin{itemize}
    \item \textbf{CLIP-RN/VIT} We include four ResNet-based contrastive CLIP models (ResNet-50, ResNet-101, ResNet50x4, ResNet50x64) and three ViT-based CLIP models (ViT/B-32, ViT/B-16, ViT/L-14).
    \item \textbf{Vision Transformer (ViT)} We include ViT models pretrained on ImageNet-1K~\cite{steinerhow,dosovitskiy2020image}(ViT/S-16, ViT/B-16, ViT/B-16@384px) and ImageNet-21K(ViT/T-16, ViT/S-16, ViT/B-16, ViT/L-16).
    \item \textbf{BIT-M} We use four ResNetv2-based model pretrained in ImageNet21K: ResNetv2-50, ResNetv2-50x3, ResNetv2-101, RseNetv2-152x4.
    \item \textbf{ResNet} We use three ResNet models prerained in ImageNet-1K: ResNet-50, ResNet-101, ResNet-152.
    \item \textbf{MoCov3} We include the self-supervised pretrained models with ResNet-50 and ViT/B-16 architectures.
\end{itemize}

\subsection{Method Implementation }
\paragraph{Hyperparameters}
We have training set size $n$ and  neighborhood size $k$ as hyperparameters. For main results, except for the ablation study, we use $n=10000$ across all datasets, except for BIRDS with $5994$ training samples in total. The candidate pool is sampled from the training set except for STL, for which we sample from the unlabeled set. We select $k \in \{10, 20, 50, 100, 200, 500, 1000\}$ with optimal AUROC performance on validation split for each dataset. See Table ~\ref{appendix:hyper}.

\begin{table*}[h]
    \centering
    \caption{Hyperparameters used in each dataset. Source: the data source of neighborhood samples. As STL has unlabeled samples, we sample from its unlabeled split as neighborhood sample sets. Evaluation in distribution shifts use the same setting as the corresponding ID.}
    \label{appendix:hyper}
    \begin{tabular}{llrr}
    \toprule
        Datasets & Source & $n$ & $k$  \\
            \midrule
         CIFAR10 & train & 10000 & 200 \\
         CIFAR100 & train & 10000 & 20 \\
         BIRDS & train &  5994 & 10 \\
         STL & unlabeled & 10000 & 200  \\
         FOOD & train & 10000  & 20 \\
         ImageNet & train & 10000 & 10 \\
    \bottomrule
    \end{tabular}
\end{table*}

we extract the features with different pretrained encoders and save for the later test stage. For feature extracting, we only require one-pass inference cost on training sample set, which is low computational compared to fully fine-tuning or adapting.

\subsection{Baselines}
Our baseline methods include the Maximum Softmax Probability (MSP)~\cite{hendrycks17baseline}, other uncertainty measure from the trained model: Entropy and Energy~\cite{lakshminarayanan2017simple,liu2020energy}, the distance based measure: TrustScore~\cite{jiang2018trust}, MaxLogit proposed for distribution shift ~\cite{vaze2022openset} and the vanilla Temperature Scaling (T.S.)~\cite{guo2017calibration}.

\subsection{Evaluation Metrics}

Following the evaluation in ~\cite{hendrycks17baseline}, we treat success/error prediction as positive and negative respectively, and use 
the area under the receiver operating characteristic curve (AUROC) as evaluation metric, with a bigger value indicating a more accurate failure detection.

\section{Framework}

\begin{algorithm}
\caption{Inter-model Latent Agreement}\label{alg:framework}
\begin{algorithmic}
    \State \textbf{Input:} 
    Training dataset $D_{tr}$, 
     validation dataset $D_{val}$, the encoder $B$ from a trained classifier, the pretrained encoders $ \mathcal{H} = \{ H_1, \dots , H_m \}$ from $m$ foundation models, test sample $\x^{\rm test}$ 
    \State \textbf{Output:} the adjusted probabilities: $\tilde{P}(Y \mid \x^{\rm test} )$ 
    
    \State Collect feature vectors with encoders $B$ and $H_1, \dots, H_m$ based on $D_{tr}$ as $\mathbb{Z}$ and $\mathbb{H}^1, \dots, \mathbb{H}^m$. \Comment{Eq\eqref{eq:z_vec}\eqref{eq:h_vec}}
    \State \textbf{Calibration:}
        \State For $\x$ in $D_{val}$, we compute $\mathsf{AS}(\x, B, \mathcal{H})$ \Comment{Eq\eqref{eq:agreement_score}} \\
        Obtain $\tau^*(\x)$ by minimizing the NLL loss on $D_{val}$  \Comment{Eq\eqref{eq:sample_temp}\eqref{eq:TS}}
    \State \textbf{Test Stage:}
        \State Given a test sample $\x^{\rm test}$, we compute $\mathsf{AS}(\x^{\rm test}, B, \mathcal{H})$ \Comment{Eq\eqref{eq:agreement_score}}
    \State \textbf{Return:} the adjusted probabilities: $\tilde{P}(Y \mid \x^{\rm test} )$ with $\tau^*(\x^{\rm test})$  \Comment{Eq\eqref{eq:TS}}
        
\end{algorithmic}
\end{algorithm}

\section{Failure Detection Results}

\subsection{Ablation Study}

We show the numeric results of ablation study on hyperparameters and choices of agreement measures in Table ~\ref{tab:ablation_hyper} and \ref{tab:ablation_score}, respectively.

\begin{table}[h]
    \centering
    \caption{Ablation study on different $k$ and $n$}
    \label{tab:ablation_hyper}
    \scalebox{.9}{
    \begin{tabular}{|l|l|l|l|l|l|l|l|}
    \toprule
        ~ & MSP & T.S. & N=2000 & N=5000 & N=10000 & N=20000 & N=50000 \\ 
        \midrule
        CIFAR10 & 94.15 & 93.76 & 97.26 & 97.40 & 97.45 & 97.53 & 97.50 \\ 
        CIFAR100 & 88.42 & 87.21 & 89.08 & 89.58 & 89.63 & 90.16 & 90.38 \\ 
        \bottomrule
    \end{tabular}}
\end{table}

We can see that even using a small subset of training data (e.g. N=2000, $<5\%$ training data from CIFAR10/CIFAR100) can already perform better than the baseline methods. The numeric results show that our method works well without a large amount of training data as the pool. Generally, the performance improves with larger pool sizes but stabilizes around an N of 10000 to 20000. Note that, with varying $n$, we select $k \in \{10, 20, 50, 100, 200, 500, 1000\}$ with optimal AUROC performance on validation split, which is following our main result experimental protocol.

\begin{table}[h]
    \centering
    \caption{Ablation study on different choices of agreement measures. The performance is averaged over all ID datasets. }
    \label{tab:ablation_score}
    \scalebox{.9}{
    \begin{tabular}{lr}
        \toprule
         Method & Avg AUROC(\%)\\
         \midrule
         MSP & 91.11\\
         T.S. & 90.98\\
         T.S. (w/ Spearmanr, single) & 90.87\\
         T.S. (w/ CKA, single) & 90.99\\
         T.S. (w/ Jaccard, single) & 92.59\\
         T.S. (w/ $\mathsf{AS}$, single) & \textbf{92.67} \\
         \bottomrule
    \end{tabular}}
\end{table}

We replace NDCG with other similarity measures: Spearman's rank correlation coefficient, Centered Kernel Alignment (CKA)~\cite{kornblith2019similarity} and Jaccard Similarity between $k$-hop neighborhood sets. For fair comparison, we also use $k$-hop neighborhood samples to compute CKA. The linear CKA and RBF-kernel CKA report similar results in our experiments. In Table ~\ref{tab:ablation_score}, the results imply the importance of neighborhoods, adaptivity to more general transformations, and ranking information, by the comparison with Spearman, CKA, and Jaccard respectively. Note that, the adaptivity to more general transformations is most important among these factors as NDCG and Jaccard outperform CKA and Spearman, which might be credited to the fact that NDCG is invariant to more general transformations compared to CKA.

\subsection{Exploration Study}
\label{appendix:fail_detect}

Note that, the exploration study about KNN Accuracy and Correlation heatmap is conducted on CNN-based trained classifiers across different datasets.

\begin{figure*}[h]
    \centering
    \begin{subfigure}{0.33\textwidth}
        \includegraphics[width=\textwidth]{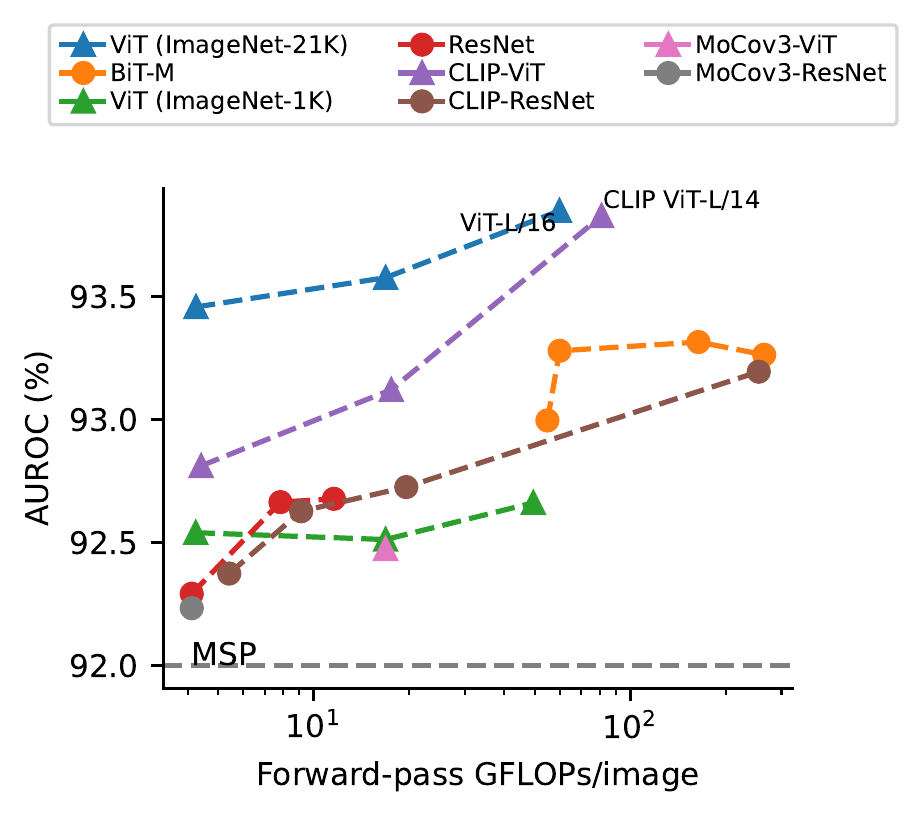}
        \caption{Average over CNN and ViT}
    \end{subfigure}
    \begin{subfigure}{0.33\textwidth}
        \includegraphics[width=\textwidth]{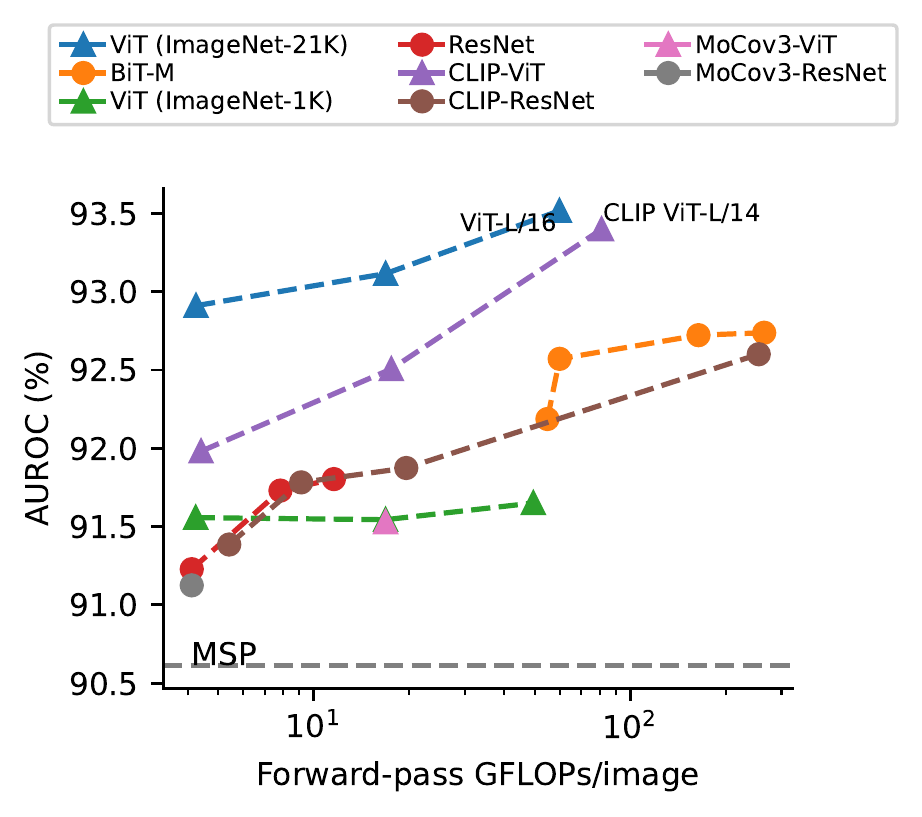}
        \caption{CNN Classifiers}
    \end{subfigure}
    \begin{subfigure}{0.33\textwidth}
        \includegraphics[width=\textwidth]{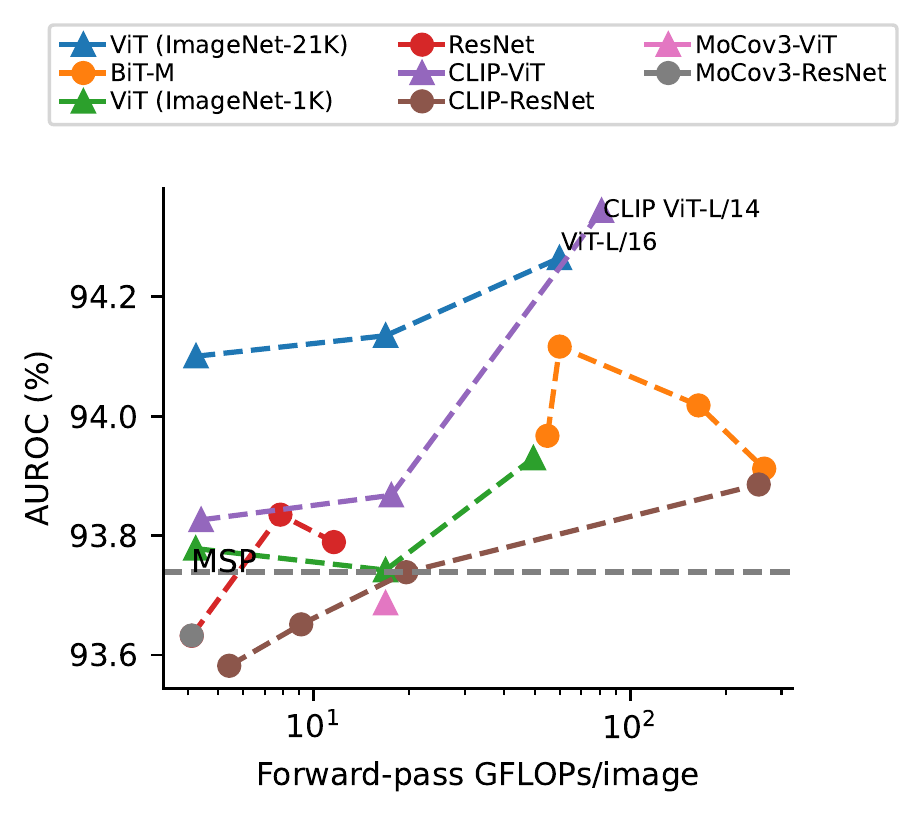}
        \caption{ViT Classifiers}
    \end{subfigure}
    \caption{Failure detection performance AUROC(\%) for each pretrained model average over all ID datasets (ImageNet excluded). x-axis: inference GPU cost. }
    \label{fig:flops_vs_perform_series}
\end{figure*}

\begin{figure*}
    \centering
    \begin{subfigure}{0.33\textwidth}
        \includegraphics[width=\textwidth]{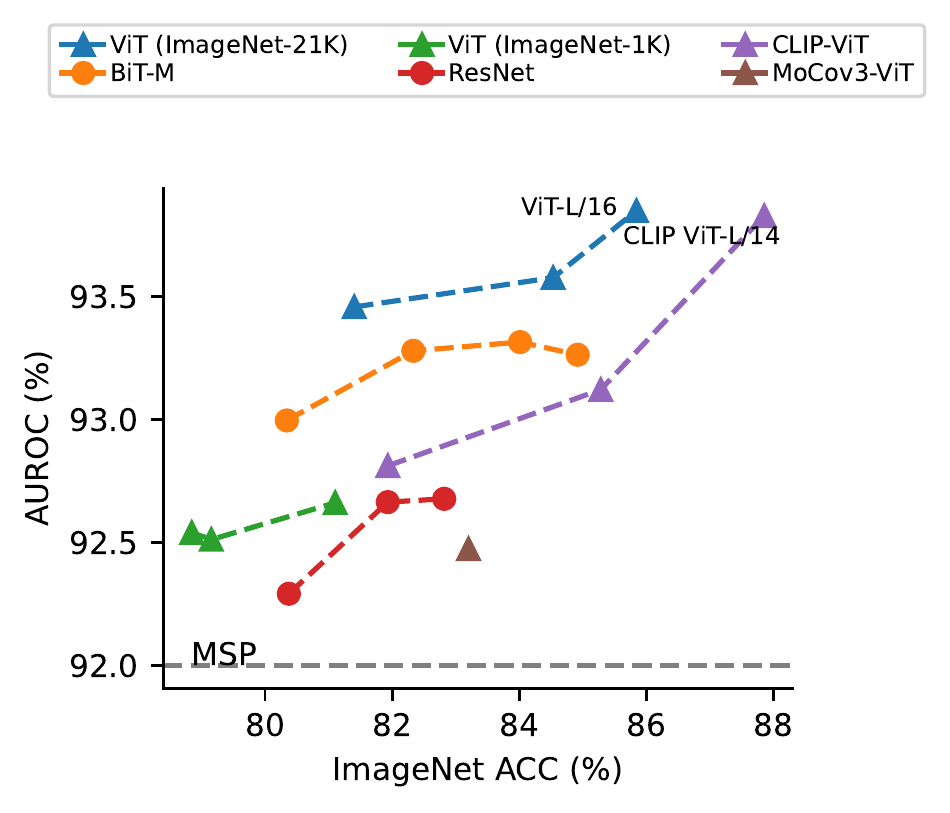}
        \caption{Average over CNN and ViT}
    \end{subfigure}
    \begin{subfigure}{0.33\textwidth}
        \includegraphics[width=\textwidth]{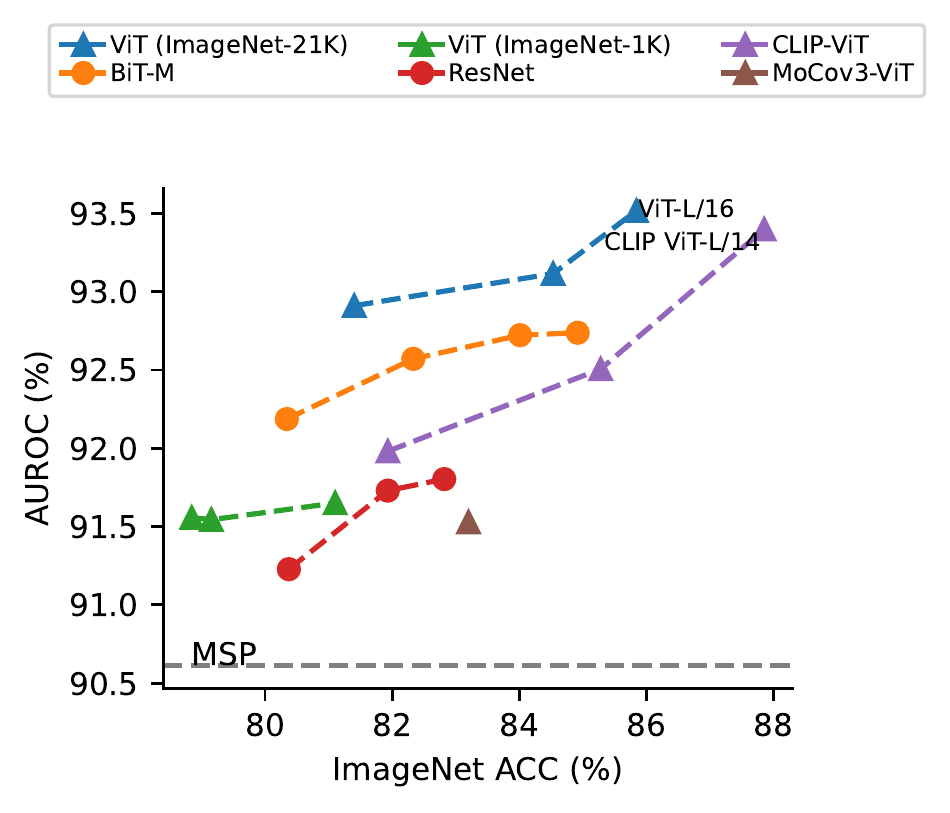}
        \caption{CNN Classifiers}
    \end{subfigure}
    \begin{subfigure}{0.33\textwidth}
        \includegraphics[width=\textwidth]{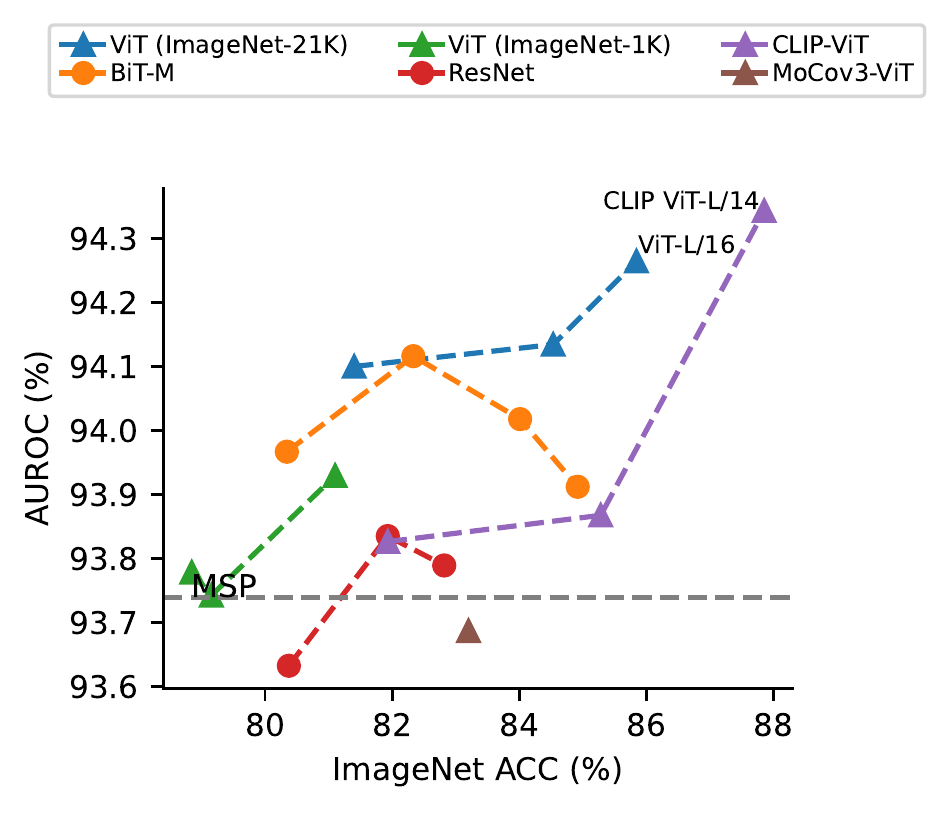}
        \caption{ViT Classifiers}
    \end{subfigure}
    \caption{Failure detection performance AUROC(\%) for each pretrained model average over all ID datasets (ImageNet excluded). x-axis: finetune ImageNet accuracy (\%) of pretrained model.}
    \label{fig:imagenet_acc_vs_perform_series}
\end{figure*}

\begin{figure*}
    \centering
    \begin{subfigure}{0.33\textwidth}
        \includegraphics[width=\textwidth]{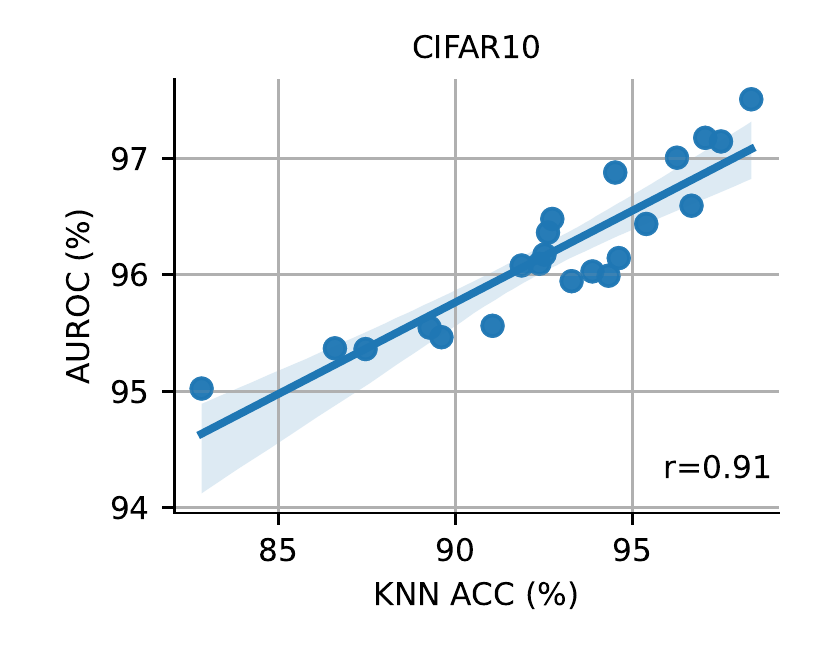}
    \end{subfigure}
    \begin{subfigure}{0.33\textwidth}
        \includegraphics[width=\textwidth]{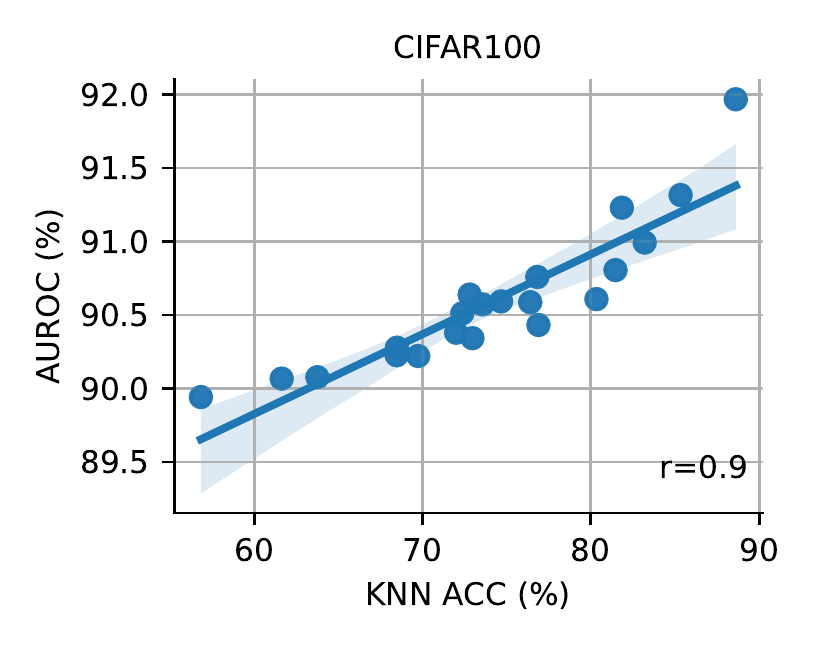}
    \end{subfigure}
    \begin{subfigure}{0.33\textwidth}
        \includegraphics[width=\textwidth]{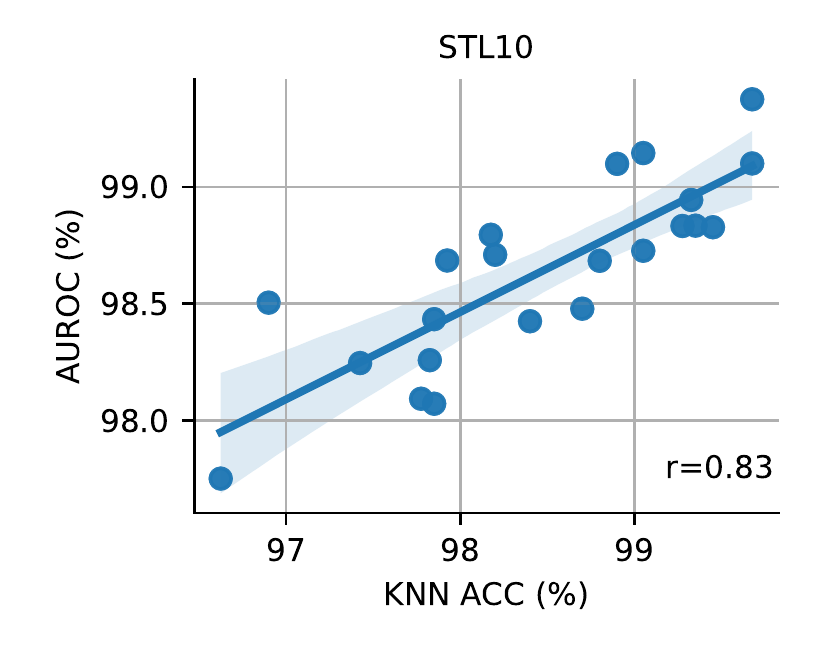}
    \end{subfigure}

    \begin{subfigure}{0.33\textwidth}
        \includegraphics[width=\textwidth]{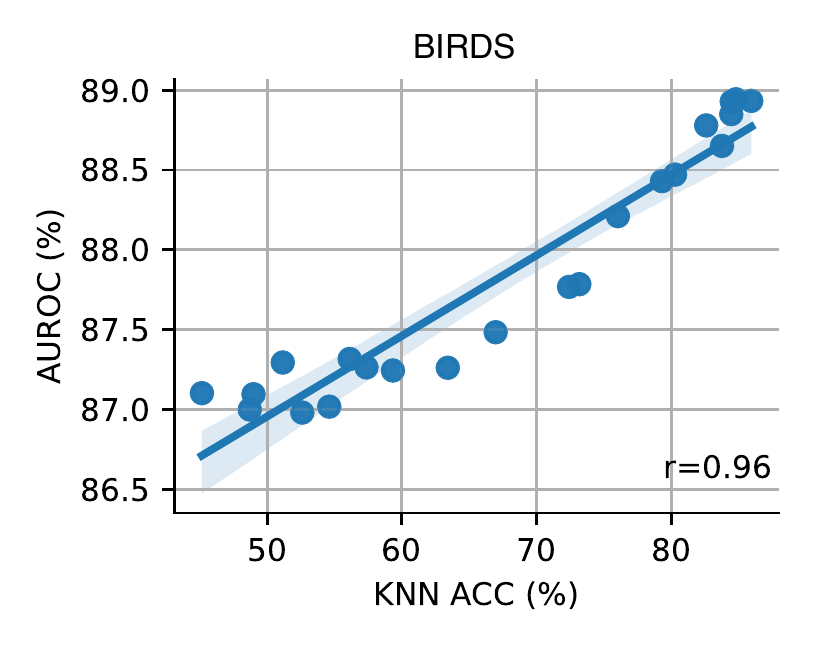}
    \end{subfigure}%
    \begin{subfigure}{0.33\textwidth}
        \includegraphics[width=\textwidth]{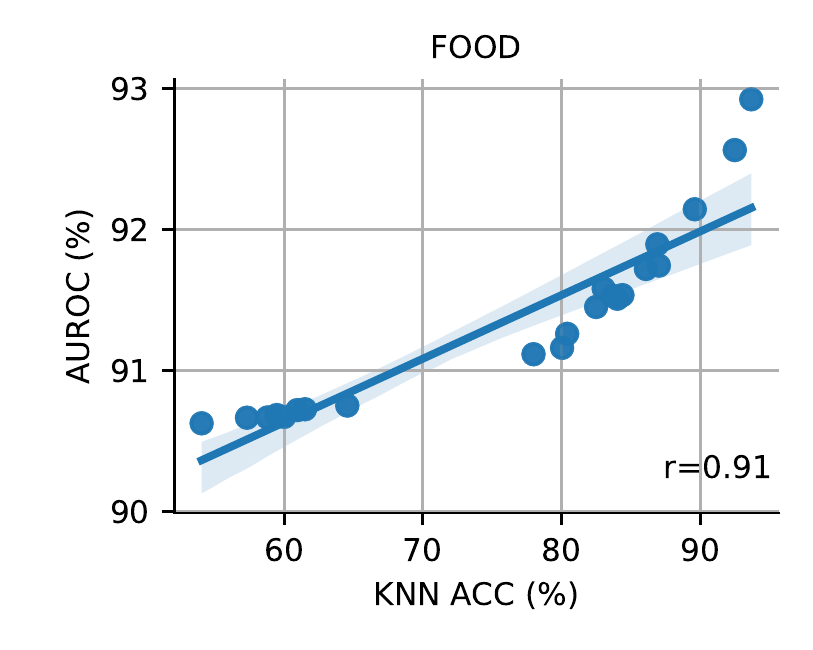}
    \end{subfigure}

    \caption{Strong correlation between Failure detection performance AUROC(\%) for pretrained model and its KNN accuracy performance in each dataset.}
    \label{tab:knn_acc_vs_per_all}
\end{figure*}

\begin{figure*}
    \centering
    \begin{subfigure}{0.33\textwidth}
        \includegraphics[width=\textwidth]{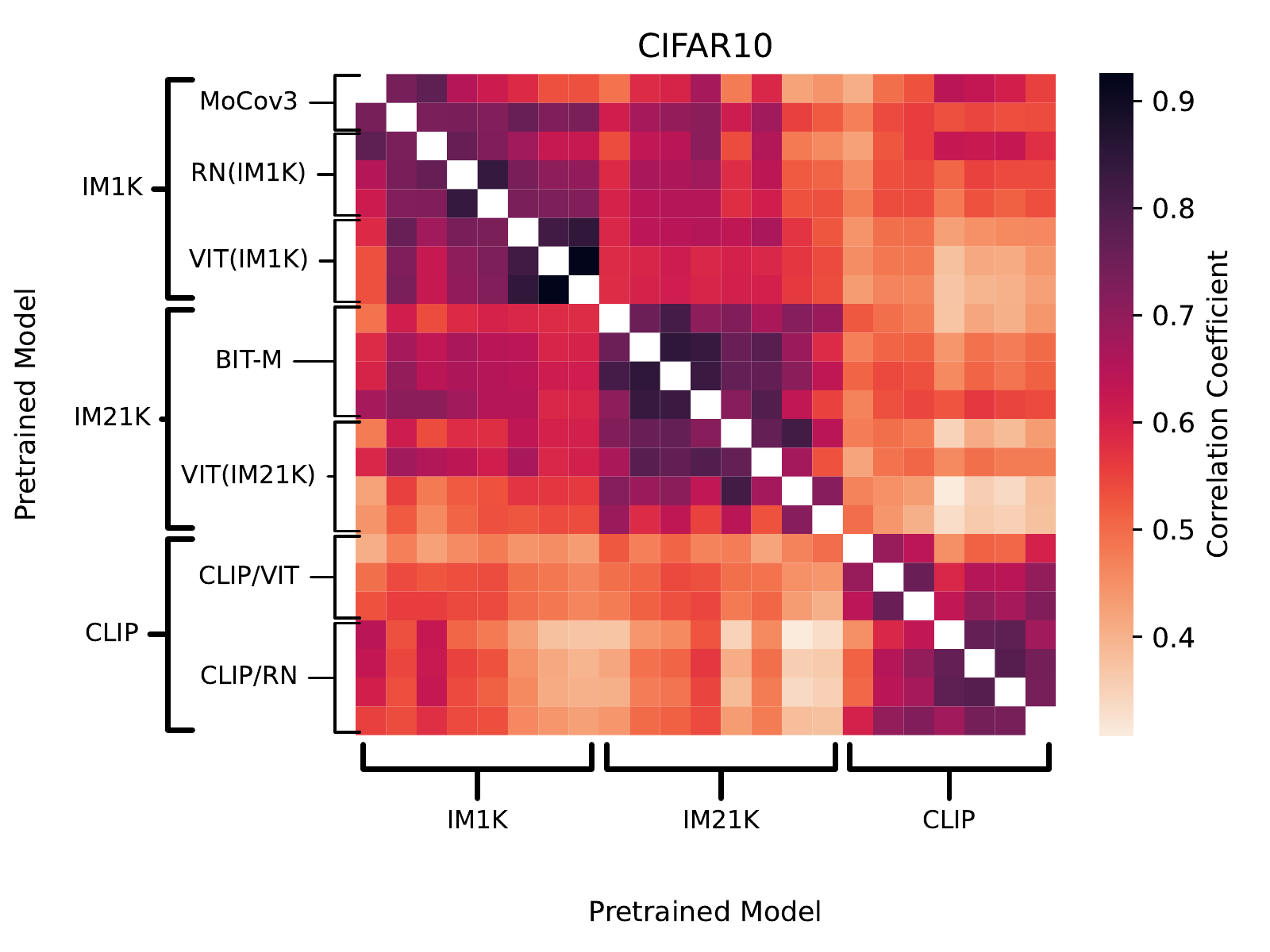}
    \end{subfigure}
    \begin{subfigure}{0.33\textwidth}
        \includegraphics[width=\textwidth]{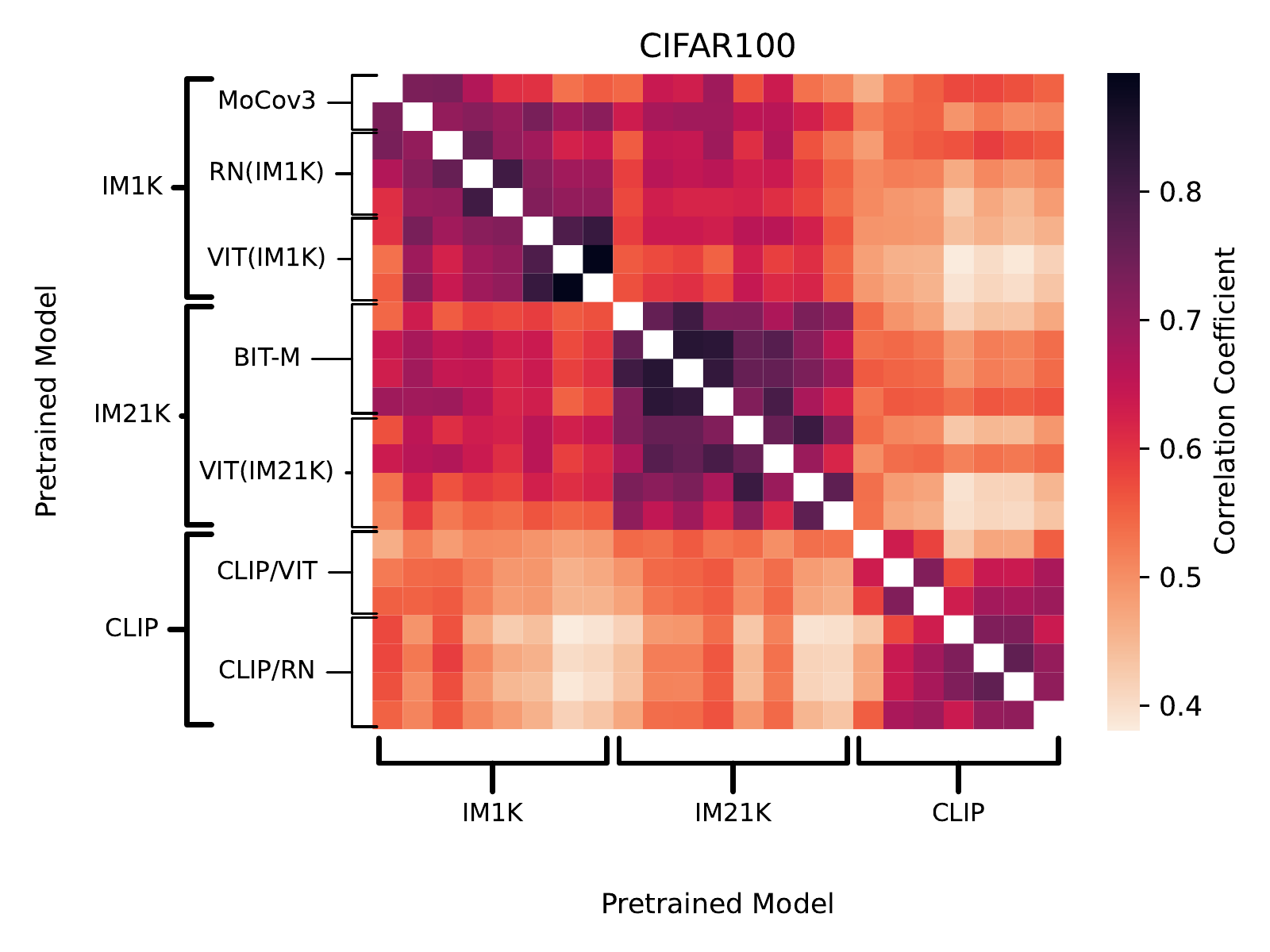}
    \end{subfigure}
    \begin{subfigure}{0.33\textwidth}
        \includegraphics[width=\textwidth]{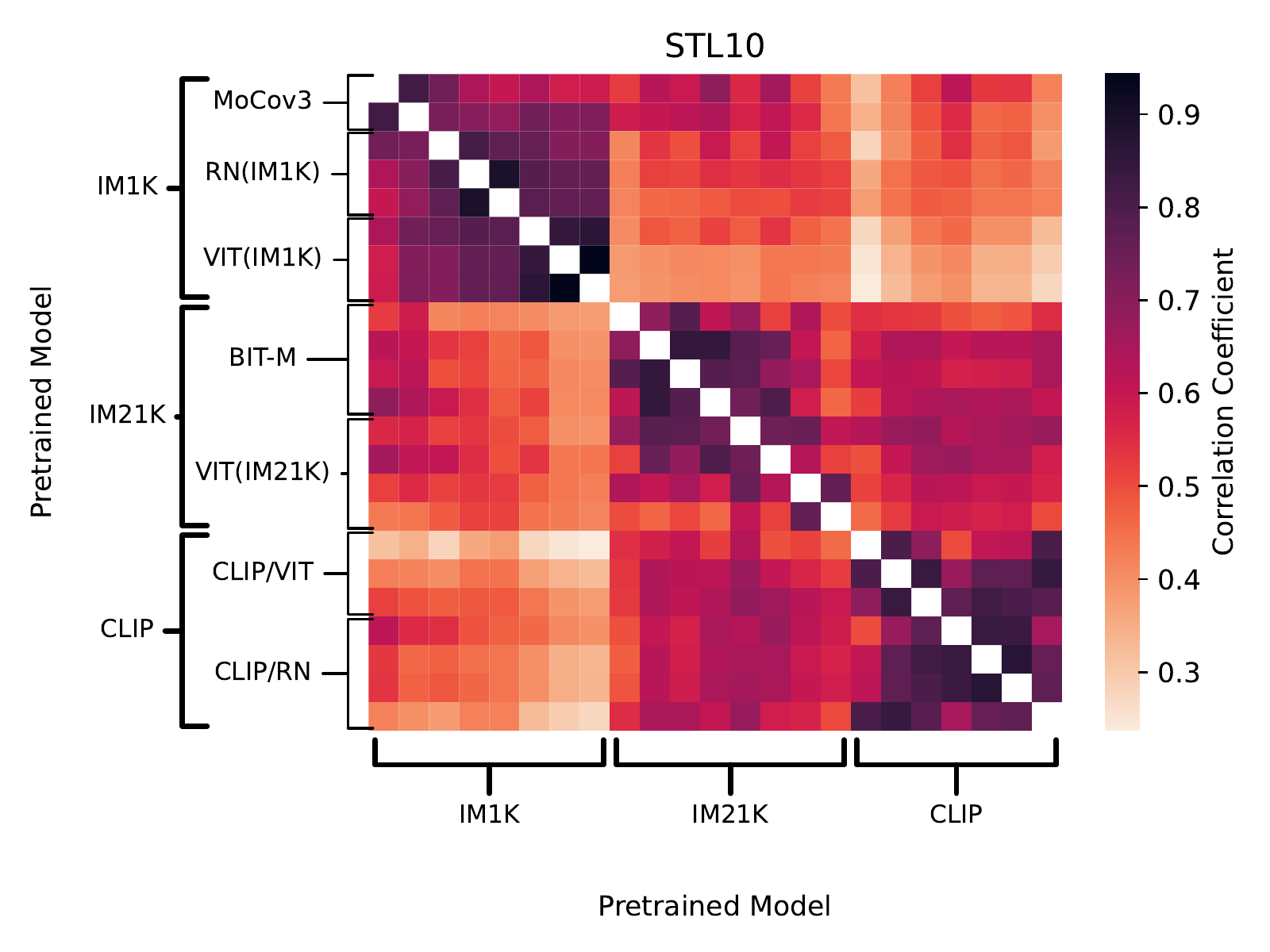}
    \end{subfigure}

    \begin{subfigure}{0.33\textwidth}
        \includegraphics[width=\textwidth]{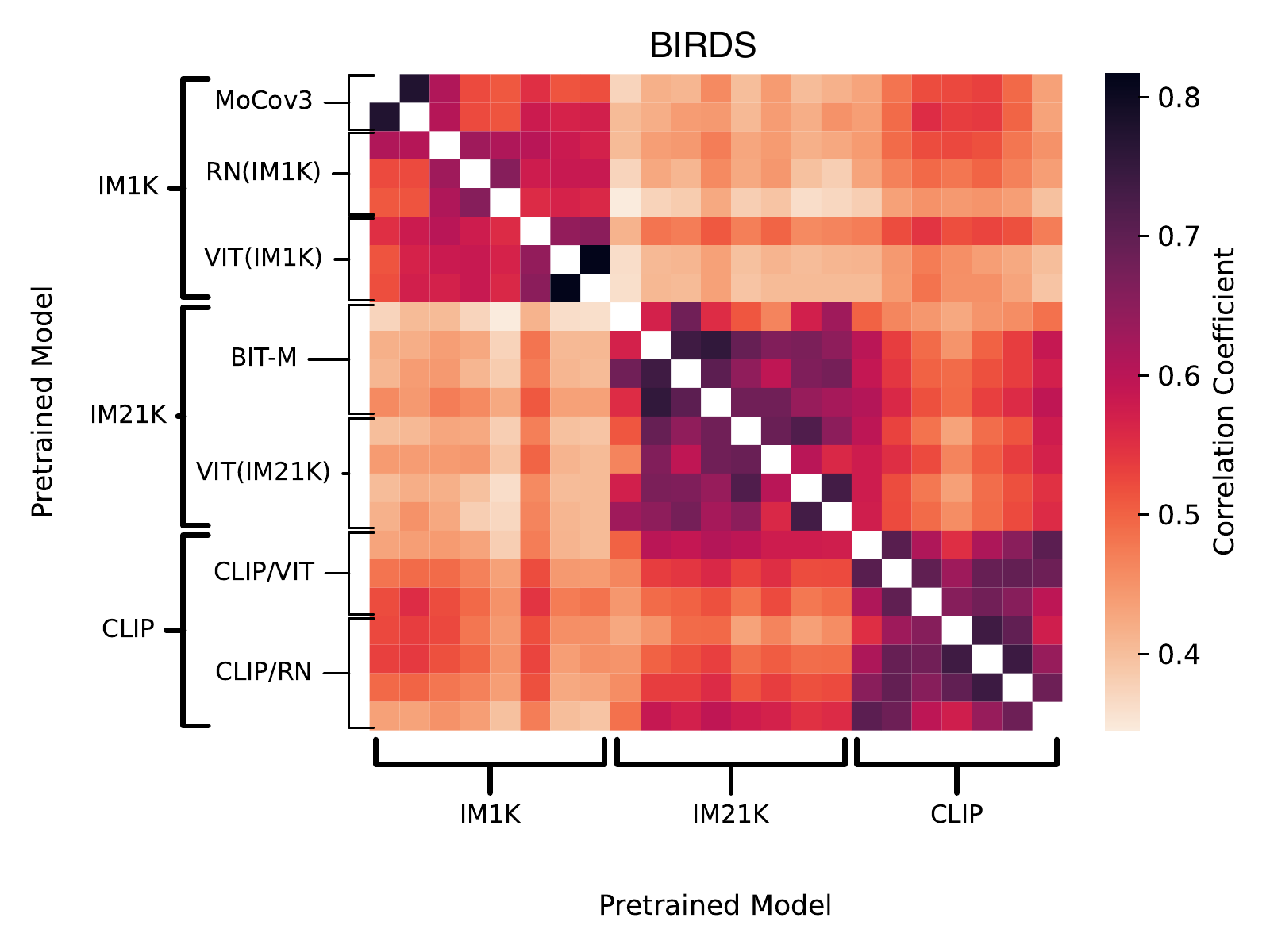}
    \end{subfigure}
    \begin{subfigure}{0.33\textwidth}
        \includegraphics[width=\textwidth]{FIG/heatmap/de_correlaton_food.pdf}
    \end{subfigure}

    \caption{Correlation heatmap inter different pre-trained model families.}
    \label{fig:cor_heatmap_all}
\end{figure*}

\clearpage

\section{Proof of Theoretical Analysis}

\paragraph{Setup} we discuss the cases under a regression problem. Let the training dataset contains $N$ samples, $D^\prime = \{\x^{(i)}, y^{(i)}\}_{i=1}^N$, where $ \mathbf{x}^{(i)} \in  \mathbb{R}^d$ is the $i$-th input sample and $y^{(i)}$ is the ground-truth scalar value. We can denote the predictor $f_{w,B}(x) = \mathbf{w}^\top B(\x) $, which consists of a feature extractor with normalization $B : \mathbb{R}^d \to \mathbb{R}^k$, and a weight vector $\mathbf{w} \in \mathbb{R}^{k \times 1}$. Thus, given the training dataset $D^\prime$, one can obtain a well-trained predictor $f_{w_0, B_0}(x)$ by minimizing a loss function $l$, \ie $\mathbf{w}_0, B_0 = \argmin_{\mathbf{w}, B} l(\x, y, \mathbf{w}, B)$, where the loss function can be squared loss, \etc. We use $\Vert \cdot \Vert$ as $\ell_2$ norm. Let $H: \mathbb{R}^d \to \mathbb{R}^k$ be the pretrained encoder for a foundation model. 

 We denote $l(\x, y, B, \w) \coloneqq \Vert \w^\top B(\x) - y \Vert $. Let $\w_0 \coloneqq \argmin_{\w} \frac{1}{n} \sum_{\x}{ \left\Vert \w^{\top}B_0(x) - y \right\Vert } $. We assume there exists an rotation transformation $U_h \in \mathcal{U} $, where  $\mathcal{U}$ contains all possible rotation transformations and $U_h \coloneqq \argmin_{U} \mathbb{E} \left\Vert B_0(X) - U H(X)  \right\Vert $. We assume that there exists a head $\w_h \in \mathbb{R}^{k \times 1}$, where $ \left\Vert \w_h^\top U_h^{-1} - \w_0^\top \right\Vert = \left\Vert \Delta \right\Vert \leq C $ . We use normalized features, \ie $\left\Vert B_0(\x) \right\Vert = 1$.

\begin{proposition}
Given a test sample $\x$ and its ground-truth value $y$, the trained encoder $B_0$ and its linear head $\w_0$, we have a pretrained encoder $H$ from a foundation model. If the pretrained encoder predicts correctly: $ l(\x, y, H, \w_h) = 0 $, the prediction error $l(\x, y, B_0, \w_0) \leq  ( C + \left\Vert \w_0 \right\Vert) \cdot \left\Vert B_0(\x) - U_h H(\x) \right\Vert + C$.
\end{proposition}

\begin{proof}\label{appendix:proof_test_error}

\begin{align}
    l(\x, y, B_0, \w_0) &= \left\Vert \w_0^\top B_0(\x) - y \right\Vert \\
    &= \left\Vert (\w_h^\top U_h^{-1} - \Delta) (U_h H(\x) + B_0(\x) - U_h H(\x))  - y  \right\Vert \\
    &= \left\Vert \w_h^\top U_h^{-1} U_h H(\x) - y + \w_h^\top U_h^{-1}(B_0(\x) - U_h H(\x))) - \Delta B_0(\x)  \right\Vert \\
    &= \left\Vert \w_h^\top U_h^{-1}(B_0(\x) - U_h H(\x)) - \Delta B_0(\x) \right\Vert \\
    & \stackrel{(a)}{\leq} \left\Vert \w_h^\top \right\Vert \cdot \left\Vert B_0(\x) - U_h H(\x) \right\Vert + C\\
    &\leq ( C + \left\Vert \w_0 \right\Vert) \cdot \left\Vert B_0(\x) - U_h H(\x) \right\Vert + C,
\end{align}
where $(a)$ comes from $\Vert \Delta \Vert \leq C$ and $\Vert U_h \Vert = 1$.
\end{proof}
In summary, that is if two latent spaces highly agree on a sample $\x$, \ie $\left\Vert B_0(\x) - U_h H(\x) \right\Vert$ close to $0$, the model is more likely to predict accurately on $\x$.

Recall that if a transformation $f$ can approximately preserve the distance around $x$ after the transformation, we call this transformation $\delta$-Local Approximation Isometry:
\begin{assumption}
($\delta$-Local Approximation Isometry). $\forall \z \in \mathcal{N}_k(\x), \exists \delta \geq 1, \frac{ \left\Vert f(\z) - f(\x) \right\Vert }{ \left\Vert \z - \x \right\Vert } \in \left( \frac{1}{\delta}, \delta \right)$.
\end{assumption}

\begin{proposition}
(Lower Bound of NDCG scores).
Given an input sample $\x$, $\Pi^*$ and $\Pi^\prime$ are permutations before and after a $\delta$-local approximation isometric transformation $f$, we have $\mathsf{NDCG}( \Pi^*, \Pi^\prime, r ) \geq \frac{1}{\delta^2}$, when $r(\cdot) = 1/d(\cdot, \x)$ and $d$ is a distance scoring function.
\end{proposition}

\begin{proof}\label{appendix:proof_ndcg_bound}
    \begin{align}
        \mathsf{NDCG}(\Pi^*, \Pi^\prime, r) = 
         \frac{ \sum_{i}^n \frac{ r({\Pi^\prime_{(i)}}) } { \log{(i+1)} }}
         {{\sum_{i}^n \frac{ r({\Pi^*_{(i)}}) } { \log{(i+1)} }} }
         &= \frac{ \sum_{i}^n \frac{ 1/d( \x_{\Pi^\prime_{(i)}}, \x ) } { \log{(i+1)} }  }
         {{\sum_{i}^n \frac{ 1/d(\x_{\Pi^*_{(i)}}, \x) } { \log{(i+1)} }}  } \\
         &\stackrel{(a)}{\geq} \frac{ \sum_{i}^n \frac{ 1/d(f( \x_{\Pi^\prime_{(i)}}), f(\x) ) } { \log{(i+1)} } \cdot \frac{1}{\delta} }
         {{\sum_{i}^n \frac{ 1/d(f(\x_{\Pi^*_{(i)}}), f(\x)) } { \log{(i+1)} }} \cdot \delta } \\
         &\stackrel{(b)}{=} \frac{ \sum_{i}^n \frac{ r^\prime({\Pi^\prime_{(i)}}) } { \log{(i+1)} } \cdot \frac{1}{\delta} }
         {{\sum_{i}^n \frac{ r^\prime({\Pi^*_{(i)}}) } { \log{(i+1)} }} \cdot \delta }\\
         &\stackrel{(c)}{\geq} \frac{1}{\delta^2},
    \end{align}
    where $r^\prime(\cdot) = 1/d(f(\cdot), f(\x))$, $(a)$ comes from the definition of $\delta$-local approximation isometry, $(b)$ is substituting $r'$ and $(c)$ comes from the Rearrangement Inequality - note that in the numerator, the terms $r^{\prime}(\Pi_{(i)}^{\prime})$ are arranged in non-increasing order with $i$, since $\Pi’$ is defined as the permutation which sorts samples based on their distances to $x$ after applying the $f$ function, i.e. the distances defining the score $r^{\prime}$. As such, in the numerator, both $r^{\prime}(\Pi_{(i)}^{\prime})$ and $1/\log(i+1)$ are arranged in the same order:
$r^{\prime}(\Pi_{(1)}^{\prime}) \geq \dots \geq  r^{\prime}(\Pi_{(n)}^{\prime})$ and $1/\log (i+1) \geq  \dots \geq  1/\log(n+1)$. In contrast, in the denominator, the same $r^{\prime}$ terms are present but in a (possibly) different order, so the denominator is smaller than or equal to the numerator by the Rearrangement Inequality. Equality is achieved when the two permutations, $\Pi^*$ and $\Pi^{\prime}$, are the same.

\end{proof}

Intuitively, it shows that when the mapping between two spaces $f$ is a $\delta$-local approximation isometric transformation, if $\delta$ approaches 1 (i.e. more distance-preserving / similar between the two spaces), the proposed metric based on NDCG is guaranteed to be high ($\geq 1/\delta^2$), implying that our score function accurately reflects how similar the two latent spaces are.